\documentclass[11pt]{article}
\pdfoutput=1

\usepackage{amsmath,amssymb}
\usepackage{bm}
\usepackage{natbib}
\usepackage[usenames]{color}
\usepackage{amsthm}

\usepackage{multirow} 
\usepackage{enumerate}
\usepackage{bbm}

\usepackage{svg}
\allowdisplaybreaks

\usepackage[colorlinks,
linkcolor=red,
anchorcolor=blue,
citecolor=blue
]{hyperref}

\usepackage{setspace}
\usepackage[left=1in, right=1in, top=1in, bottom=1in]{geometry}

\usepackage{xcolor}
\usepackage{footnote}

\usepackage{mylatexstyle}

\usepackage[utf8]{inputenc} 
\usepackage[T1]{fontenc}    
\usepackage{hyperref}       
\usepackage{url}            
\usepackage{booktabs}       
\usepackage{multirow}
\usepackage{amsfonts}       
\usepackage{nicefrac}       
\usepackage{microtype}      

\ifdefined\final
\usepackage[disable]{todonotes}
\else
\usepackage[textsize=tiny]{todonotes}
\fi

\title{\huge Provable Robustness of Adversarial Training\\ for Learning Halfspaces with Noise}

\begin{document}

\author
 {   
     Difan Zou\thanks{Equal Contribution } \thanks{Department of Computer Science, UCLA; e-mail: {\tt knowzou@cs.ucla.edu}} 
 	~~~and~~~
 	Spencer Frei$^*$\thanks{Department of Statistics, UCLA; e-mail: {\tt spencerfrei@ucla.edu}} 
 	~~~and~~~
 	Quanquan Gu\thanks{Department of Computer Science, UCLA; e-mail: {\tt qgu@cs.ucla.edu}
 	}
}

\date{}
            


\def\supp{\mathop{\text{supp}}}
\def\card{\mathop{\text{card}}}
\def\rank{\mathrm{rank}}
\def\tr{\mathop{\text{tr}}}
\newcommand{\red}{\color{red}}
\newcommand{\blue}{\color{blue}}
\newcommand{\la}{\langle}
\newcommand{\ra}{\rangle}
\newcommand{\cIs}{\cI_{\hat{s}}}
\def \CC {\textcolor{red}}
\def \err {\text{err}}
\def \hbtheta{ \btheta}
\def \barell{\bar \ell}
\def \Breg {B_{\psi}}
\def \tbtheta {\tilde{\btheta}}
\def \poly
{\mathrm{poly}}
\def \tr {\text{Tr}}
\def \oneb {\boldsymbol{1}}
\def \proj {\mathbf \Pi}
\def \poproberr {\err_{\calD}^{p,r}}
\def \OPT {\mathsf{OPT}}

\newcommand{\opt}{\mathsf{OPT}}
\def \bpxr {\cB_p(\xb,r)}
\def \bpxir {\cB_p(\xb_i,r)}
\def \bdelta {\boldsymbol{\delta}}

\maketitle 

\begin{abstract}
    We analyze the properties of adversarial training for learning adversarially robust halfspaces in the presence of agnostic label noise.  Denoting $\opt_{p,r}$ as the best robust classification error achieved by a halfspace that is robust to perturbations of $\ell_p$ balls of radius $r$, we show that adversarial training on the standard binary cross-entropy loss yields adversarially robust halfspaces up to (robust) classification error $\tilde O(\sqrt{\opt_{2,r}})$ for $p=2$, and $\tilde O(d^{1/4} \sqrt{\opt_{\infty, r}} + d^{1/2} \opt_{\infty,r})$ when $p=\infty$.   Our results hold for distributions satisfying anti-concentration properties enjoyed by log-concave isotropic distributions among others.   We additionally show that if one instead uses a nonconvex sigmoidal loss, adversarial training yields halfspaces with an improved robust classification error of $O(\opt_{2,r})$ for $p=2$, and $O(d^{1/4} \opt_{\infty, r})$ when $p=\infty$.  To the best of our knowledge, this is the first work to show that adversarial training provably yields robust classifiers in the presence of noise.  
\end{abstract}

\section{Introduction}
Modern deep learning models are powerful but brittle: standard stochastic gradient descent (SGD) training of deep neural networks can lead to remarkable performance as measured by the classification accuracy on the test set, but this performance rapidly degrades if the metric is instead \textit{adversarially robust} accuracy.  This brittleness is most apparent for image classification tasks \citep{szegedy2014.intriguing,goodfellow2015.explainharnessadversarial}, where neural networks trained by gradient descent achieve state-of-the-art classification accuracy on a number of benchmark tasks, but where imperceptible (adversarial) perturbations of an image can force the neural network to get nearly all of its predictions incorrect.

To formalize the above comment, let us define the robust error of a classifier.   Let $\calD$ be a distribution over $(\xb, y)\in \R^d \times \{\pm 1\}$, and let $f : \R^d \to \{\pm 1\}$ be a hypothesis classifier.  For $p\in [1,\infty]$ and perturbation radius $r>0$, the \textit{$\ell_p$ robust error for radius $r$} is given by
\begin{equation}
    \poproberr(f) = \PP_{(\xb,y)\sim\cD}\big[ \exists \xb' : \norm{\xb - \xb'}_p\leq r, \text{ and }y \neq f(\xb')\big]\label{eq:robust.acc.classifier}
\end{equation}
The standard accuracy of a classifier $f$ is given by $\err_\calD(f) = \PP_{(\xb ,y)\sim \calD} (y \neq f(\xb))$, and is equivalent to the robust accuracy at radius $r=0$.  That SGD produces neural networks $f$ with high classification accuracy but low robust accuracy means that $\err_\calD(f) \approx 0$ but $\poproberr(f) \approx 1$, even when $r$ is an extremely small number.  

The vulnerability of SGD-trained neural networks to adversarial examples has led researchers to introduce a number of methods aimed at improving the robustness of neural networks to adversarial examples~\citep{kurakin2017adversarial,madry2018.resistantadversarial,tramer2018.ensembleadversarial,zhang2019principledtradeoffrobustness,wang2019convergence,wang2019improving}.  One notable approach is known as \textit{adversarial training}, where the standard SGD algorithm is modified so that data samples are perturbed $\xb \mapsto \xb + \deb$ with the aim of increasing the robust accuracy.  In the same way that one minimizes the standard classification error by minimizing a surrogate loss, adversarial training seeks to minimize
\begin{equation}\label{eq:robust.loss}
    \ldrob(f) = \EE_{(\xb, y)\sim \calD} \sup_{\xb': \norm{\xb'-\xb}_p \leq r} \ell(y f(\xb')),
\end{equation}
where $\ell(\cdot)$ is some convex surrogate for the 0-1 loss.  Unfortunately, the inner maximization problem is typically intractable, especially when $f$ comes from a neural network function class.  Indeed, it is often difficult to calculate \textit{any} nontrivial upper bound for the robust loss $\sup_{\xb':\norm{\xb-\xb'}_p\leq r} \ell(y f(\xb'))$ for a fixed sample $\xb$.  A number of recent works have focused on developing upper bounds for the robust loss that are computationally tractable, which enables end-users to certify the robustness of learned classifiers by evaluating the upper bound on test samples~\citep{raghunathan2018certified,wong2018provable,cohen2019certified}.  Additionally, upper bounds for the robust loss can then be used as a new objective function to be minimized as an alternative to the intractable robust loss.  This approach has seen impressive results in improving the adversarial robustness of classifiers, but unfortunately these procedures do not come with a provable guarantee that the learned classifiers will be adversarially robust. 
To the best of our knowledge, only two works have been able to show that the standard gradient-based adversarial training of \eqref{eq:robust.loss} provably yields classifiers with a guarantee on the robust (population-level) classification error: \citet{charles2019adversarial} and \citet{li2019adversarial}.  Both of these papers considered the hypothesis class of halfspaces $\xb \mapsto \sign(\wb^\top \xb)$ and assumed that the data distribution is linearly separable by a hard margin $\gamma_0>0$, so that for some $\wb \in \RR^d$, $y \wb^\top \xb \geq \gamma_0 > 0$ holds almost surely over $\calD$.  

In this work, we show that adversarial training provably leads to halfspaces that are approximate minimizers for the population-level robust classification error.  In particular, adversarial training provably yields classifiers which are robust even when the data is not linearly separable.  Let us denote the best-possible robust classification error for a halfspace as
\begin{align*}
\OPT_{p,r} = \min_{\|\wb\|_q=1} \poproberr(\wb),
\end{align*}
where $\poproberr(\wb)$ is the robust error induced by the halfspace classifier.  Our main contributions are as follows.
\begin{enumerate}
    \item We show that adversarial training on the robust surrogate loss \eqref{eq:robust.loss} yields halfspaces with $\ell_2$ robust error at most $\tilde O(\sqrt{\opt_{2,r}})$ when $\ell$ is a typical convex surrogate loss and $\calD$ satisfies an anti-concentration property enjoyed by log-concave isotropic distributions.  For $p=\infty$, our guarantee is $\tilde O(d^{1/4} \sqrt{\opt_{\infty, r}} + d^{1/2} \opt_{\infty, r})$.
    \item When $\ell$ is a nonconvex sigmoidal loss, the guarantees for adversarial training improves to
    $O\big(d^{\frac{1}{4}-\frac{1}{2p}}\|\wb_*\|_2^{1/2}\OPT_{p,r}\big)$ for $\ell_p$ perturbations, where $\wb^*$ of norm $\|\wb^*\|_q=1$ (for $1/p+1/q=1$) is the optimal model. This implies that adversarial training achieves $O(\opt_{2,r})$ robust error for perturbations in the $\ell_2$ metric, and $O(d^{1/4} \opt_{\infty, r})$ robust error when $p=\infty$ in the worst case.  
\end{enumerate}
To the best of our knowledge, these are the first results that provide a guarantee that adversarial training will generate adversarially robust classifiers on noisy data distributions.

\subsection{Additional Related Work}
Adversarial training and adversarial examples have attracted significant attention recently due to the explosion of research in deep learning, but the broader problem of learning decision rules that are robust to perturbations of the data has appeared in a number of forms.  One of the main motivations for support vector machines is to maximize the margin of the classifier, which can be understood as a form of robustness to perturbations of the input~\citep{rosenblatt1958perceptron,boser1992svm}.  Robust optimization is a field in its own right dedicated to the analysis of optimization algorithms that are robust to perturbations of the algorithms' inputs~\citep{bental2009.robustoptim}.  Adversarial examples have been studied in the context of spam filtering, where it was observed that spam prevention algorithms could be bypassed with small modifications to the text of an email~\citep{dalvi2004adversarial,lowd2005adversarial,lowd2005good}.    We refer the reader to the survey of~\cite{biggio2018adversarialsurvey} for a more detailed history of adversarial machine learning.

Following the first paper on adversarial examples in deep learning~\citep{szegedy2014.intriguing}, a sequence of works sought to develop empirical methods for improving the robustness of neural network classifiers~\citep{goodfellow2015.explainharnessadversarial,papernot2016distillationadversarial}.  These proposed defenses against adversarial examples were quickly defeated by more sophisticated attacks~\citep{carlini2017evaluatingrobustness}.  This led a number of authors to develop \textit{certifiable} defenses against adversarial attacks, where one can prove that the defense algorithm will be robust to adversarial perturbations~\citep{wong2018provable,raghunathan2018certified}.  These works typically derive an upper bound for the robust loss that can be computed exactly and then introduce optimization procedures for minimizing the upper bound.  This allows for one to certify whether or not a classifier is provably robust to adversarial perturbations for a given sample.  But since the procedure is based upon minimizing an upper bound for the desired error, there is no guarantee that every classifier which is trained using this procedure will (provably) yield a classifier that has nontrivial robust classification accuracy.  

In terms of provable guarantees for learning adversarially robust classifiers, adversarial training was shown to yield provably robust halfspace classifiers by~\citet{charles2019adversarial} and \citet{li2019adversarial} under the assumption that there exists a robust classifier with perfect accuracy that separates the data by a large margin.  A separate approach for developing robust classifiers is known as randomized smoothing~\citep{salman2019provablerobustsmoothing,lecuyer2018certified,cohen2019certified}, where one can convert a base classifier into a robust classifier by smoothing out the predictions of the base classifier over Gaussian noise perturbations of the input.~\citet{gao2019adversarial} and~\citet{zhang2020adversarial} showed that adversarial training with multilayer neural networks leads to classifiers with small robust training loss, but were not able to translate these into guarantees for small test (population-level) robust error. 
~\citet{montasser2020adversarialnoise} showed that the standard gradient descent algorithm on the (non-robust) empirical risk using a convex margin loss yields halfspaces that are robust in the presence of random classification noise.\footnote{Random classification noise (RCN) is a generalization of the realizable setting, where an underlying halfspace $y = \sign(\wb^\top \xb)$ has labels flipped with probability $p$.  By contrast, in the adversarial label noise setting we consider in this paper, one makes no assumptions on the relationship between $\xb$ and $y$.}~\citet{diakonikolas2020complexity} studied the computational complexity of learning robust halfspaces in the agnostic noise setting.

We wish to emphasize that in this work we are interested in developing \textit{computationally efficient} algorithms for learning adversarially robust halfspaces in the presence of noise.  
\cite{cullina2018adversarialvc} recently developed a notion of adversarial VC dimension, which allows for a characterization of the number of samples necessary to learn robust classifiers in the presence of noise by analyzing the robust empirical risk minimizer (ERM).  However, the non-convexity of the zero-one loss makes the task of finding a robust ERM a highly non-trivial task.  Indeed, it is known that no polynomial time algorithm can agnostically learn standard (non-robust) halfspaces up to risk $O(\opt_{p,0})+\eps$ without distributional assumptions~\citep{daniely2016complexity}, although standard VC dimension arguments show that $\poly(d, \eps^{-1})$ samples suffice for the ERM to achieve $\opt_{p,0}+\eps$ risk.   Thus, in order to develop computationally efficient algorithms that can robustly learn up to robust risk $O(\opt_{p,r})$, we must make assumptions on the distribution.  

There are a number of other important questions in adversarial robustness for which a detailed review is beyond the scope of this paper.  We briefly note that some related topics include understanding the possible tradeoffs between robust accuracy and non-robust accuracy~\citep{zhang2019principledtradeoffrobustness,tsipras2019robustnessacc,javanmard2020precise,raghunathan2020tradeoffrobustnessacc,yang2020accuracy.robustness,wu2020does};  what types of features robust classifiers depend upon~\citep{ilyas2019adversarialbugs}; and the transferability of robust classifiers~\citep{salman2020robusttransfer}.

\subsection{Notation}
We use bold-faced letters to denote vectors.  For a scalar $x$, we use $\text{sgn}(x)\in\{+1,-1\}$ to denote its sign. For $p\in [1,\infty]$, we denote $\bpxr = \{ \xb': \norm{\xb - \xb'}_p\leq r\}$ as the $\ell_p$ ball of radius $r$ centered at $\xb$.  We use $\cS^{d-1}_q$ to denote the unit $\ell_q$ sphere. Given two vectors $\wb$ and $\vb$, we use $\angle(\wb,\vb)$ to denote the angle between these two vectors. We use the indicator function $\ind(\mathfrak{E})$ to denote $1$ on the event $\mathfrak{E}$ and $0$ elsewhere.   We use the standard $O(\cdot)$ and $\Omega(\cdot)$ notations to hide universal constants, with $\tilde O(\cdot)$ and $\tilde \Omega(\cdot)$ additionally ignoring logarithmic factors.  The notation $g(x) = \Theta(f(x))$ denotes a function with growth rate satisfying both $g(x) = O(f(x))$ and $g(x) = \Omega(f(x))$.  

\subsection{Paper Organization}
The remainder of the paper is organized as follows.  In Section \ref{sec:convex}, we describe our guarantees for adversarial training on convex loss functions.  In Section \ref{sec:nonconvex}, we show that by using a nonconvex sigmoidal loss, we can achieve improved guarantees for the robust classification accuracy of halfspaces.  We conclude in Section \ref{sec:conclusion}. 

\section{Adversarial Training with Convex Surrogates}\label{sec:convex}
Our first set of results is for the case that the loss function $\ell$ appearing in the definition of the robust loss~\eqref{eq:robust.loss} is a typical decreasing convex surrogates of the zero-one loss, such as the cross entropy $\ell(z) = \log(1+\exp(-z))$ or hinge loss $\ell(z) = \max(0, 1-z)$.  We consider a standard approach for gradient descent-based adversarial training of the objective \eqref{eq:robust.loss}, which consists of two parts: (1) an inner maximization, and (2) an outer minimization.  For the inner maximization, we find the optimal perturbation of the input which maximizes $\ell(y \wb^\top(\xb + \deb))$ for $\deb \in \cB_p(0,r)$.  For more complicated model classes, such as neural networks, the inner maximization procedure can often be very difficult to optimize.  As such, it is usually difficult to derive provable guarantees for the robustness of adversarial training procedures. However, in the linear model class that we consider, we can solve the inner maximization procedure exactly.  When $\ell$ is decreasing, this maximization problem is equivalent to
\begin{align*}
\arg\min_{\norm{\deb}_p\leq r} y\wb^\top(\xb + \deb).
\end{align*}
Using calculus we can solve for the exact solution to this minimization problem.  The optimal perturbation is given by $\deb^* = \deb^*(\wb, r, y)$, with components
\begin{align}\label{eq:delta.star.def}
\delta^*_j = -ry\cdot\text{sgn}(w_j)|w_j|^{q-1}/\|\wb\|_q^{q-1},
\end{align}
where $q$ is the H\"older conjugate to $p$ so that $1/q+1/p=1$.
The ability to solve the inner maximization procedure exactly means that the only remaining part is to solve the outer minimization.  For this, we use the standard gradient descent algorithm on the perturbed examples.  We note that we do not differentiate through the samples in the gradient updates---although the perturbed examples depend on the weights (via $\deb^*$), we treat these perturbed samples as if they are independent of $\wb$.  The update rule is explicitly given in Algorithm \ref{alg:adv_training}.

\begin{algorithm}[!t]
	\caption{Adversarial Training}
	\label{alg:adv_training}
	\begin{algorithmic}[1]
		\STATE \textbf{input:} 
		Training dataset $\cS = \{(\xb_i, y_i)\}_{i=1,\dots,n}$, step size $\eta$
 		\FOR {$k = 0,1,\ldots, K$}
 		\FOR {$i = 1, \dots, n$} 
 		\STATE $\dit := \argmax_{\norm{\deb}_p\leq r} \ell(y_i \wb_k^\top(\xb_i + \deb))$
 		\ENDFOR
		\STATE  $\wb_{k+1} = \wb_k - \frac{\eta}{n} \summ i n \ell'(y_i w^\top(\xb_i + \dit))y_i (\xb_i + \dit)$
		\ENDFOR 
		\STATE \textbf{output: $\{\wb_k\}_{k=0,\dots,K}$} 
	\end{algorithmic}

\end{algorithm}
Our first result is that Algorithm~\ref{alg:adv_training} efficiently minimizes the robust empirical risk.

\begin{lemma}\label{lemma:convergence_advtraining}
Assume $\ell$ is convex, decreasing, and $1$-Lipschitz.  Let $\wb^*\in \RR^d$ be arbitrary.  Let $p\in [1,\infty]$, and assume that $\norm{\xb}_p\leq 1$ a.s.  If $p\leq 2$, let $H = 4$, and if $p > 2$, let $H = 4d$.  Let $\eps>0$ and be arbitrary.  If $\eta \leq \eps H^{-1}/4$, then for any initialization $\wb_0$, if we denote $\wb_k$ as the $k$-th iterate of Algorithm \ref{alg:adv_training}, by taking $K = \eps^{-1}\eta^{-1} \norm{\wb_0 - \wb^*}_2^2$, we have there exists a $k^*\le K$ such that $\|\wb_{k^*}-\wb^*\|_2^2\le \|\wb_0-\wb^*\|_2^2$ and 
\[ \lsrob(\wb_{k^*}) \leq \lsrob(\wb^*) + \eps.\]
\end{lemma}
The proof for the above Lemma can be found in Appendix \ref{appendix:cross.ent.empirical.risk}. 
To convert the guarantee for the empirical risk into one for the population risk, we will utilize an argument based on robust Rademacher complexity~\citep{yin2019rademacher}.  This is possible because Lemma \ref{lemma:convergence_advtraining} shows that the weights returned by Algorithm \ref{alg:adv_training} stay in a norm-bounded region.  
\begin{lemma}[Population robust loss]\label{lemma:generalization_GD}
Assume that $\norm{\xb}_p\leq 1$ a.s. and that $\ell$ is convex, decreasing, and 1-Lipschitz.  Let $\wb^*\in \RR^d$ be such that $\norm{\wb^*}_q \leq \rho$ for some $\rho >0$, and denote $B = \ell(0) + 2d^{|1/q-1/2|}(1+r)\rho$ and $\bar B = 2d^{|1/q-1/2|}B$.  Denote $\mathfrak{R}_p = n^{-1}\EE_\sigma\big[\|\sum_{i=1}^n\sigma_i\xb_i\|_p\big]$. 
Then for any $\eps>0$, using the same notation from Lemma \ref{lemma:convergence_advtraining}, running Algorithm \ref{alg:adv_training} with $\wb_0=0$ ensures that there exists $k^*\leq K= \max\{1,d^{1/q-1/2}\}\eta^{-1} \eps^{-1} \rho^2$ such that with probability at least $1-\delta$,
\begin{align*}
L_{\cD}^{p,r}(\wb_{k^*}) &\le L_{\cD}^{p,r}(\wb^*)+ \epsilon+4\bar B\rho \mathfrak{R}_p + 4\bar B\frac{\rho r}{\sqrt{n}} + 6B\sqrt{\frac{\log(2K/\delta)}{2n}}
\end{align*}
\end{lemma}
The proof for Lemma \ref{lemma:generalization_GD} is in Appendix \ref{appendix:cross.ent.generalization}.  We note that the term $\mathfrak{R}_p$ is a common complexity term that takes the form $O(1/\sqrt{n})$ for $p=2$ and $O(\log(d)/\sqrt{n})$ for $p=\infty$; see e.g. Lemmas 26.10 and 26.11 of~\citet{shalev2014understanding}. 

Now that we have shown that adversarial training yields hypotheses which minimize the \textit{surrogate} robust risk $\ldrob$, the next step is to show that this minimizes the robust classification error $\poproberr$.  (Since $\ldrob$ is only an upper bound for $\poproberr$, minimizers for $\ldrob$ do not necessarily minimize $\poproberr$.)  Recently,~\citet{frei2020halfspace} introduced the notion of soft margins in order to translate minimizers of surrogate losses to approximate minimizers for classification error, and we will use a similar approach here.  Let us first define soft margin functions. 
\begin{definition}\label{def:softmargin}
Let $q \in [1,\infty]$. Let $\bar \vb\in \R^d$ satisfy $\norm{\bar \vb}_q = 1$.  We say $\bar \vb$ satisfies the $\ell_q$ \emph{soft margin condition with respect to a function $\phi_{\bar \vb,q}:\R \to \R$} if for all $\gamma \in [0,1]$, it holds that
\[ \E_{\xb\sim \calD_x}\left[\ind\left(   |\bar \vb^\top \xb| \leq \gamma \right)\right]\leq \phi_{\bar \vb,q}(\gamma).\]
\end{definition}
The properties of the $\ell_q$ soft margin function for $q=2$ for a variety of distributions were shown by~\citet{frei2020halfspace}.   We collect some of these in the examples below, but let us first introduce the following definitions which will be helpful for understanding the soft margin. 
\begin{definition}\label{def:anticoncentration}
For $\bar \vb, \bar \vb'\in \R^d$, denote by $p_{\bar \vb,\bar \vb'}(\cdot)$ the marginal distribution of $\xb \sim \calD_x$ on the subspace spanned by $\bar \vb$ and $\bar \vb'$.  We say $\calD_x$ satisfies \emph{$U$-anti-concentration} if there is some $U>0$ such that for any two vectors $\bar \vb, \bar \vb'$ satisfying $\norm{\bar \vb}_2=\norm{\bar \vb'}_2 =1$, we have $p_{\bar \vb, \bar \vb'}(\zb)\leq U$ for all $\zb\in \mathbb{R}^2$.  We say that $\calD_x$ satisfies \emph{$(U',R)$-anti-anti-concentration} if there exists $U', R>0$ such that $p_{\bar \vb, \bar \vb'}(\zb )\geq 1/U'$ for all $\zb \in \mathbb{R}^2$ satisfying $\norm{\zb}_2\leq R$.
\end{definition}
Anti-concentration and anti-anti-concentration have recently been used for deriving guarantees agnostic PAC learning guarantees for learning halfspaces~\citep{diakonikolas2020massartstructured,diakonikolas2020nonconvex,frei2020halfspace}.  Log-concave isotropic distributions, such as the standard Gaussian in $d$ dimensions or the uniform distribution over any convex set, satisfy $U$-anti-concentration and $(U',R)$-anti-anti-concentration with each of $U$, $U'$, and $R$ being universal constants independent of the dimension of the input space.  Below, we collect some of the properties of the $\ell_2$ soft margin function.
\begin{example}\label{example:soft.margin.l2}
\begin{enumerate}
    \item For any $q\in [1,\infty]$, if $\bar \vb \in \R^d$ satisfies $\norm{\bar \vb}_q=1$ and $|\bar \vb^\top \xb| > \gamma^*$ a.s., then $\phi_{\bar \vb, q}(\gamma) = 0$ for $\gamma < \gamma^*$.
    \item If $\calD_x$ satisfies $U$-anti-concentration, then $\phi_{\bar \vb, 2}(\gamma) = O(\gamma)$.  
    \item If $\calD_x$ satisfies $(U', R)$-anti-anti-concentration, then for $\gamma\leq R$, $\phi_{\bar \vb, 2}(\gamma) = \Omega(\gamma)$ holds.
    \item Isotropic log-concave distributions (i.e. isotropic distributions with log-concave probability density functions) satisfy $U$-anti-concentration and $(U', R)$-anti-anti-concentration for $U, U', R = \Theta(1)$.  
\end{enumerate}
\end{example}
Proofs for these properties can be found in Appendix \ref{appendix:soft.margin}.  For $q\neq 2$, the soft margin function will depend upon the ratio of the $\ell_q$ to the $\ell_2$ norm, since we have the identity, for any $\bar \vb$ satisfying $\norm{\bar \vb}_q=1$,
\begin{align}
    \phi_{\bar \vb, q} (\gamma)&= \PP(|\bar\vb^\top \xb | \leq \gamma) = \PP\left( \frac{|\bar\vb^\top\xb|}{\norm{\bar\vb}_2} \leq \frac{\gamma}{\norm{\bar\vb}_2} \right)= \phi_{\bar \vb/\norm{\bar \vb}_2,2}(\gamma / \norm{\bar\vb}_2).\label{eq:soft.margin.lq.intermediate}
\end{align}
Thus, the $\ell_q$ soft margin function scales with the ratio of $\norm{\bar \vb}_q/\norm{\bar \vb}_2$.  The case $q=1$ corresponds to the $\ell_\infty$ perturbation and is of particular interest.  By Cauchy--Schwarz, $\norm{\bar \vb}_1\leq \sqrt{d}\norm{\bar \vb}_2$, and this bound is tight in the worst case (take $\bar \vb = \mathbf{1}\in \RR^d$).  Thus the $\ell_1$ soft margin has an unavoidable dimension dependence in the worst case.  We collect the above observations, with some additional properties that we show in Appendix \ref{appendix:soft.margin}, in the following example.
\begin{example}\label{example:soft.margin.lq}
\begin{enumerate}
\item If $\calD_x$ satisfies $U$-anti-concentration, and if $q \in [1,2]$, then for any $\bar \vb\in \RR^d$ with $\|\bar \vb\|_q=1$, $\phi_{\bar \vb, q}(\gamma) = O(\gamma d^{\frac 1q - \frac 12})$.  
\item If $\calD_x$ satisfies $(U', R)$-anti-anti-concentration and if $q \in [1,2]$, then for any $\bar \vb\in \RR^d$ with $\|\bar \vb\|_q=1$, for $\gamma \leq \Theta(R)$, it holds that $\phi_{\bar \vb, q}(\gamma) = \Omega(\gamma)$.  
\end{enumerate}
\end{example}

We now can proceed with relating the minimizer of the surrogate loss $\ldrob$ to that of $\poproberr$ by utilizing the soft margin.  The proof for the following Lemma is in Appendix \ref{appendix:surrogate.vs.classification.robust}.
\begin{lemma}\label{lemma:adversarial.training.robust.risk.surrogate}
Let $p,q\in [1,\infty]$ be such that $1/p+1/q=1$ and assume $\norm{\xb}_p\leq 1$ a.s.  Let $\bar \vb := \min_{\| \wb\|_q = 1} \poproberr (\wb)$, so that $\poproberr(\bar \vb) = \opt$.   Assume that $\|\xb\|_p\leq 1$ a.s.  For $\rho>0$, denote $\vb := \rho \bar \vb$ as a scaled version of the population risk minimizer for $\poproberr(\cdot)$.  Assume $\ell$ is $1$-Lipschitz, non-negative and decreasing.  Then we have
\begin{align}
    \ldrob(\vb) &\leq \inf_{\gamma >0} \Big \{ (\ell(0) + \rho ) \opt_{p,r}+ \ell(0) \phi_{\bar \vb,q}(r + \gamma) + \ell(\rho \gamma) \Big \}.\label{eq:surrogate.vs.classification.key}
\end{align}
Thus, if $\ell(0)>0$, 
\begin{align*}
\poproberr(\vb) &\leq [\ell(0)]^{-1} \inf_{\gamma >0} \Big \{ (\ell(0) + \rho ) \opt_{p,r}  + \ell(0) \phi_{\bar \vb,q}(r + \gamma) + \ell(\rho \gamma) \Big \}.
\end{align*}
\end{lemma}

Using Lemmas \ref{lemma:adversarial.training.robust.risk.surrogate} and \ref{lemma:generalization_GD}, we can derive the following guarantee for the robust classification error for classifiers learned using adversarial training.

\begin{theorem}\label{thm:crossentropy}
Suppose $\ell\geq 0$ is convex, decreasing, and 1-Lipschitz.  Let $p\in [1,\infty]$ and $q\in [1,\infty]$ satisfy $1/p+1/q=1$.  Denote $H = 4$ if $p\leq 2$ and $H = 4d$ if $p>2$.  Let $\eps>0$ be arbitrary, and fix $\eta \leq \eps H^{-1}/4$.  For any $\gamma>0$, running Algorithm \ref{alg:adv_training} with $\wb_0=0$ for $K = \max\{1,d^{\frac 2q -1}\}\eps^{-1} \eta^{-1}  \ell^{-2}(1/\eps)\gamma^{-2}$ iterations, with probability at least $1-\delta$, there exists $k^*\leq K$ such that 
\begin{align*}
\poproberr(\wb_{k^*}) &\leq   \left(1 + [\ell(0)]^{-1} \cdot \ell^{-1}(\nicefrac 1\eps) \cdot \gamma^{-1} \right) \opt_{p,r}+ \phi_{\bar \vb, q}(r + \gamma) + [\ell(0)]^{-1}\eps \\
&\qquad + 4 [\ell(0)]^{-1}  \bar B \gamma^{-1} \ell^{-1}(\nicefrac 1\eps)\mathfrak{R}_p+ [\ell(0)]^{-1} \Bigg[ \frac{4 \bar B \gamma^{-1} \ell^{-1}(\nicefrac 1\eps) r}{\sqrt n} + 6 B \sqrt{\frac{\log(\nicefrac {2K}\delta)}n}\Bigg],
\end{align*}
where $B = \ell(0) + 2d^{|2-q|/2}(1+r)\gamma^{-1} \ell^{-1}(\eps)$,  $\mathfrak{R}_p = n^{-1} \E_{\sigma_i \iid \mathrm{Unif}(\pm 1)}[\|\summ i n \sigma_i \xb_i\|_p]$, and $\bar B = 2d^{|2-q|/2}B$.
\end{theorem}
\begin{proof}
The result follows by using Lemmas \ref{lemma:generalization_GD} and \ref{lemma:adversarial.training.robust.risk.surrogate} with the choice of $\rho = \gamma^{-1} \ell^{-1}(1/\eps)$.
\end{proof}
In order to realize the right-hand-side of the above bound for the robust classification error, we will need to analyze the properties of the soft margin function $\phi_{\bar \vb, q}$ and then optimize over $\gamma$.  We will do so in the following corollaries.  We start by considering hard margin distributions.  

\begin{corollary}[Hard margin]\label{corollary:hard.margin}
Let $p \geq 2$ and $q\in [1,2]$ be such that $1/p+1/q =1$.  Assume $\norm{\xb}_p\leq 1$ a.s. Suppose $\vb^*$ is such that $\|\vb^*\|_q =1$ and for some $\gamma_0\in [0,1]$, $|\sip{\vb^*}{\xb}| \geq \gamma_0$, and $\poproberr(\vb^*) = \min_{\norm{\wb}_q=1} \poproberr(\wb) = \opt_{p,r}$.  Suppose we consider the perturbation radius $r = (1-\nu)\gamma_0$ for some $\nu \in (0,1)$.  Consider the cross entropy loss for simplicity, and let $\eta \leq \opt_{p,r} H^{-1}/4$, where $H = 4$ if $p\leq 2$ and $H = 4d$ if $p>2$.  Then the adversarial training in Algorithm \ref{alg:adv_training} started from $\wb_0=0$ finds classifiers satisfying $\poproberr(\wb_k) = \tilde O(\nu^{-1} \gamma_0^{-1} \opt_{p,r})$ within $K = \tilde O(\eta^{-1} d^{\frac 2p -1} \gamma_0^{-2} \nu^{-2} \opt_{p,r}^{-1})$ iterations provided $n = \tilde \Omega(\gamma_0^{-4} \nu^{-2} \opt_{p,r}^{-2})$.
\end{corollary}
\begin{proof}
We sketch the proof here and leave the detailed calculations for Appendix \ref{appendix:corollaries}.  By the definition of soft margin, $\phi_{\bar \vb^*,q}(\gamma_0) = 0$, and so if we choose $\gamma = \nu \gamma_0$ and $\eps = \opt_{p,r}$ in Theorem \ref{thm:crossentropy}, we get a bound for the robust classification error of the form $\tilde O(\nu^{-1} \gamma_0^{-1} \opt_{p,r}) + \tilde O(1) \cdot \mathfrak{R}_p + \tilde O(1/\sqrt n)$ by using the fact that $\ell^{-1}(\nicefrac 1/\eps) = O(\log(\nicefrac 1/\eps))$ for the cross entropy loss.  Standard arguments in Rademacher complexity show that $\mathfrak{R}_p = \tilde O(1/\sqrt{n})$, completing the proof.
\end{proof}


The above corollary shows that if the best classifier separates the samples with a hard margin of $\gamma_0$ (including when it makes incorrect predictions), then adversarial training will produce a classifier that has robust classification error within a constant factor of the best-possible robust classification error.  This can be seen as a generalization of the results of~\citet{charles2019adversarial} and~\citet{li2019adversarial} from distributions that can achieve perfect robust classification accuracy (with a hard margin) to ones where significant label noise can be present.  

Our next result is for the class of distributions satisfying the anti-concentration properties described in Definition \ref{def:anticoncentration}. 
\begin{corollary}[Anti-concentration distributions]\label{corollary:anti.concentration}
Let $p \in [2,\infty]$ and assume $\norm{\xb}_p\leq 1$ a.s. Suppose $\calD_x$ satisfies $U$-anti-concentration and $(U', R)$-anti-anti-concentration for $U,U', R = \Theta(1)$.  Consider the cross entropy loss for simplicity, and let $\eta \leq \opt_{p,r} H^{-1}/4$, where $H = 4$ if $p\leq 2$ and $H = 4d$ if $p>2$.  Then for perturbations satisfying $r\leq R$, the adversarial training in Algorithm \ref{alg:adv_training} started from $\wb_0=0$ finds classifiers satisfying
\[ \poproberr(\wb_k) = \tilde O\big(d^{\frac 14 - \frac 1{2p}}  \sqrt{\opt_{p,r}} + d^{\frac 12 - \frac 1p} \opt_{p,r}\big),\]
within $K = \tilde O(\eta^{-1} d^{\frac{3}{2p} - \frac 3 4}\opt_{p,r}^{-3})$ iterations provided $n = \tilde \Omega(d^{1-2p}\opt_{p,r}^{-2})$.
\end{corollary}
\begin{proof}
We again sketch the proof here and leave the detailed calculations for Appendix \ref{appendix:corollaries}.  Example~\ref{example:soft.margin.lq} shows that $\phi_{\bar \vb^*,q}(a) = O(a d^{\frac 1q - \frac 12}) = O(ad^{\frac 12 - \frac 1p})$.  Anti-anti-concentration can be shown to imply that $r = O(\opt_{p,r})$, and thus $\phi_{\bar \vb^*,q}(\gamma + r) = O(\gamma d^{\frac 12 - \frac 1p}) + O(d^{\frac 12 - \frac 1p} \opt_{p,r})$. The first term is of the same order as $\gamma^{-1} \opt_{p,r}$ when $\gamma = \opt_{p,r}^{1/2} d^{\frac{1}{2p}-\frac 14}$, and results in a term of the form $\tilde O(d^{\frac 14 - \frac 1{2p}} \sqrt{\opt_{p,r})}$.  The other terms following using an argument similar to that of Corollary~\ref{corollary:hard.margin}.
\end{proof}

The above shows that adversarial training yields approximate minimizers for the robust classification accuracy for halfspaces over distributions satisfying anti-concentration assumptions.  In particular, this result holds for any log-concave isotropic distribution, such as the standard Gaussian or the uniform distribution over a convex set.

\begin{remark}
We note that although the guarantees in this section are for (full-batch) gradient descent-based adversarial training, nearly identical guarantees can be also derived for online SGD-based adversarial training.  We give the details on this extension in Appendix \ref{sec:SGD_guarantee}.
\end{remark}
\section{Adversarial Training with Nonconvex Sigmoidal Loss}\label{sec:nonconvex}
We now show that if instead of using a typical convex loss function we use a particular nonconvex sigmoidal loss, we can improve our guarantees for the robust classification error when using adversarial training.  We note that the approach of using nonconvex loss functions to derive improved guarantees for learning halfspaces with agnostic label noise was first used by~\citet{diakonikolas2020nonconvex}.  Our results in this section will rely upon the following assumption on the distribution $\calD_x$.  
\begin{assumption}\label{assump:isotropic}
\begin{enumerate}
    \item $\cD_x$ is mean zero and isotropic, i.e. its covariance matrix is the identity.
    \item $\cD_x$ satisfies $U$-anti-concentration and $(U', R)$-anti-anti-concentration, where $U, U', R = \Theta(1)$.
\end{enumerate}
\end{assumption}

The loss function we consider is defined by
\begin{align}\label{eq:nonconvex.loss.def}
\ell(z) = e^{-z/\sigma}\cdot\ind(z>0) + (2-e^{z/\sigma})\cdot \ind(z\le 0),
\end{align}
where $\sigma>0$ is a scalar factor to be specified later. 
In addition to using the loss function \eqref{eq:nonconvex.loss.def}, we additionally scale the weight vector, so that the surrogate loss we consider in this section is is 
\begin{align}
L_\cD^{p,r}(\wb) &= \EE_{(\xb,y)\sim \cD}\bigg[\sup_{\bx'\in\cB_p(\xb,r)}\ell\bigg(\frac{y \wb^\top\xb'}{\|\wb\|_q}\bigg)\bigg]\notag
\end{align}
The adversarial training algorithm that we consider for the loss function~\eqref{eq:nonconvex.loss.def} is a variant of Algorithm~\ref{alg:adv_training}, where we introduce a projection step to normalize the weights after each gradient update.  We additionally use the online stochastic gradient descent algorithm as opposed to full-batch gradient descent.  For this reason we call the algorithm we use for learning halfspaces that are robust to $\ell_p$ perturbations of radius $r$ by the name $\mathsf{PSAT}(p,r)$, which we describe in Algorithm \ref{alg:projected_advtraining}.  
We note that when $p=\infty$ or $p=2$ (i.e., $q=1$ or $q=2$, resp.), the projection can be done efficiently in $O(d)$ \citep{duchi2008efficient} or $O(1)$ time respectively. 

\begin{algorithm}[t]
	\caption{Projected Stochastic Adversarial Training ($\mathsf{PSAT}(p, r)$)}
	\label{alg:projected_advtraining}
	\begin{algorithmic}[1]
	    \STATE \textbf{input:} initial model parameter $\wb_1$ with $\|\wb_1\|_q=1$, learning rate $\eta$, perturbation limit $r$.
 		\FOR {$k = 1,\ldots, K$}
		\STATE Query data $(\xb_k,y_k)$ from data distribution $\cD$
		\STATE $\deb_k := \argmax_{\norm{\deb}_p\leq r} \ell\bigg(\frac{y_k \wb_k^\top(\xb_k + \deb)}{\|\wb_k\|_q}\bigg)$
		\STATE Update $\hat \wb_{k+1} \leftarrow \wb_k - \eta \nabla \ell\bigg(\frac{y_k\wb_k^\top (\xb_k+\deb_k)}{\|\wb_k\|_q}\bigg)$
		\STATE Project $\wb_{k+1} \leftarrow \arg\min_{\wb: \|\wb\|_q=1}\|\hat\wb_{k+1}-\wb\|_2$
		\ENDFOR 
		\STATE \textbf{output: $\wb_1, \wb_2, \dots, \wb_{K}$} 
	\end{algorithmic}
\end{algorithm}

In the below theorem we describe our guarantees for the robust classification error of halfspaces learned using Algorithm \ref{alg:projected_advtraining}.
\begin{theorem}\label{thm:guarantee_normalized_classifier}
Suppose the data distribution $\cD$ satisfies Assumption \ref{assump:isotropic}. Let $\sigma=r$ and $\wb^*=\arg\min_{\|\wb\|_q=1}\err_\cD^{p,r}(\wb)$ be the optimal model such that $\err_\cD^{p,r}(\wb^*) = \opt_{p,r}$. If $\err_{\cD}(\wb^*)=O(rd^{2/p-1})$ and $r=O\big(d^{\frac{3}{2p}-\frac{3}{4}}\big)$, then running the adversarial training algorithm $\mathsf{PSAT}(p,r)$ for $K=O\big(d\|\wb_1-\wb^*\|_2^2\delta^{-2}r^{-4}d^{\frac{1}{2}-\frac{1}{p}}\big)$ iterations, with probability at least $1-\delta$, there exists a $k^*\le K$ such that 
\begin{align*}
\err_{\cD}^{p,r}(\wb_{k^*}) = O\big(d^{\frac{1}{4}-\frac{1}{2p}}\cdot\|\wb^*\|_2^{1/2}\cdot\OPT_{p,r}\big).
\end{align*}
\end{theorem}

We note that the robust classification error achieved by adversarial training depends on the $\ell_2$ norm of the optimizer $\wb^*$, which satisfies $d^{1/2-1/p}\le\|\wb^*\|_2\le 1$ since $\|\wb^*\|_q=1$ (where $1/p+1/q=1$).  The strongest guarantees arise when $\|\wb^*\|_2=d^{1/2 - 1/p}$, which results in a robust classification error guarantee of $O(\opt_{p,r})$, while in the worst case $\|\wb^*\|_2=1$ and the guarantee is $O(d^{\frac 14-\frac{1}{2p}}\opt_{p,r})$ robust error.  Note that for $\ell_2$ perturbations, $\|\wb^*\|_2=1$ and so our guarantee is always $O(\opt_{2,r})$.

In the remainder of this section we will prove Theorem \ref{thm:guarantee_normalized_classifier}. 
A key quantity in our proof is the inner product $\wb^{*\top}\nabla L_{\cD}^{p,r}(\wb)$, where $\wb^{*}$ is the optimal robust halfspace classifier.\footnote{Here we slightly abuse the notation since in fact the gradient $\nabla L_{\cD}^{p,r}(\wb)$ is defined by $\nabla L_{\cD}^{p,r}(\wb) = \EE_{(\xb,y)\sim\cD}[\nabla \ell(y\wb^\top(\xb+\deb)/\|\wb\|_q)]$, where the gradient is only taken over $\wb$ and we do not differentiate through the perturbation $\deb$.}  To get an idea for why this quantity is important, consider the gradient flow approach to minimizing $\|\wb(t) - \wb^*\|_2^2$,  
\begin{align}\label{eq:grad.flow.init}
\frac{\dd\|\wb(t)-\wb^*\|_2^2}{\dd t} &= -\la\wb(t)-\wb^*,\nabla L_{\cD}^{p,r}(\wb(t))
\end{align}
If we denote $h(\wb, \xb) = \wb^\top \xb/\|\wb\|_q$, then we have the identity
\begin{align*}
\nabla L_{\cD}^{p,r}(\wb) =  \E_{(\xb, y)} \left[ \ell'\left (y h(\wb,\xb+\deb) \right) y \nabla_\wb h(\wb,\xb+\deb)\right],
\end{align*}
where
\begin{align}\label{eq:grad.h.formula.main}
\nabla_{\wb}h(\wb,\xb+\deb)=\bigg(\Ib - \frac{\bar \wb \wb^\top}{\|\wb\|_q^{q}}\bigg)\frac{\xb+\deb}{\|\wb\|_q},
\end{align}
where we denote the vector $\bar \wb$ as having components $\bar w_j = |w_j|^{q-1}\text{sgn}(w_j)$. Then we have $\wb^\top\nabla_{\wb}h(\wb,\xb+\deb)=0$ since $\wb^\top\bar\wb = \|\wb\|_q^q$.  In particular, substituting this into \eqref{eq:grad.flow.init}, we get
\begin{equation}\label{eq:grad_flow}
\frac{\dd\|\wb(t)-\wb^*\|_2^2}{\dd t} = \wb^{*\top} \nabla \ldrob(\wb(t)).
\end{equation}
This implies that the more negative the quantity $\wb^{*\top}\nabla L_{\cD}^{p,r}(\wb(t))$ is, the faster the iterates of gradient flow will converge to $\wb^*$. 

The key, then, is to derive bounds on the quantity 
\begin{align*}
\wb^{*\top}\nabla L_{\cD}^{p,r}(\wb)= \wb^{*\top} \E_{(\xb, y)} \left[ \ell'\left (y h(\wb,\xb+\deb) \right) y \nabla_\wb h(\wb,\xb+\deb) \right],
\end{align*}
where $\wb$ is an arbitrary vector which we will take to be the iterates of Algorithm \ref{alg:projected_advtraining}.  
The challenge here is that the prescence of agnostic label noise means there are no \textit{a priori} relationships between $\xb$ and $y$, making it unclear how to deal with the appearance of both of these terms in the expectation.  To get around this, we will use a similar high-level idea as did \citet{diakonikolas2020nonconvex}, in which we swap the label $y$ with the prediction of the optimal solution $\wb^*$. Then the inner product $\wb^{*\top}\nabla L_\cD^{p,r}(\wb)$ can be upper bounded only using the information of $\wb$, $\wb^*$, the distribution of $\xb$, and the classification error $\err_{\cD}(\wb^*)$.  The details of this calculation become more complicated since adversarial training introduces perturbations that also depend on the label: the optimal perturbation for $\wb$ is $-ry\bar\wb/\|\wb\|_q^{q-1}$.  This requires additional attention in the proof.  Finally, since we consider general $\ell_p$ perturbations, the normalization by norms with $p\neq 2$ introduces additional complications.

Let us begin with some basic calculations. We first give some general calculations which will be frequently used in the subsequent analyses.
Let $h(\wb,\xb) =  \wb^\top\xb/\|\wb\|_q$ be the prediction of the normalized classifier and denote the event 
\begin{equation}\label{eq:S.event.def}
S = \{(\xb, y): y= \text{sgn}(\wb^{*\top}\xb)\},
\end{equation}
as the data which can be correctly classified by $\wb^*$ without perturbation. We have
\begin{align*}
\nabla_{\wb}  L_{\cD}^{p,r}(\wb) & = \EE_{(\xb,y)\sim \cD}\big[\ell'(yh(\wb,\xb+\deb))y \nabla_{\wb}h(\wb,\xb+\deb)\ind(S)\big]\notag\\
&\qquad +\EE_{(\xb,y)\sim \cD}\big[\ell'(yh(\wb,\xb+\deb)) y \nabla_{\wb}h(\wb,\xb+\deb)\ind(S^c)\big].
\end{align*}
Note $\deb=-ry\bar\wb/\|\wb\|_q^{q-1}$ is the optimal $\ell_p$ adversarial perturbation corresponding to the model parameter $\wb$ and sample $(\xb, y)$.  A routine calculation shows that $yh(\wb,\xb+\deb) = y\wb^\top\xb/\|\wb\|_q - r$. Then it follows that
\begin{align*}
yh(\wb,\xb+\deb) = \left\{
\begin{array}{cc}
  \text{sgn}(\wb^{*\top}\xb)\cdot\frac{\wb^\top\xb}{\|\wb\|_q}-r   & (\xb,y)\in S, \\
   -\text{sgn}(\wb^{*\top}\xb)\cdot\frac{\wb^\top\xb}{\|\wb\|_q}-r  & (\xb,y)\in S^c,
\end{array}
\right.
\end{align*}
where for the data $(\xb,y)\in S$ we use $\text{sgn}(\wb^{*\top}\xb)$ to replace the label $y$ while for the data $(\xb,y)\in S^c$ we use $-\text{sgn}(\wb^{*\top}\xb)$ to replace $y$.
Define
\begin{align*}
    g_S(\wb^*,\wb;\xb) &= \ell'(\text{sgn}(\wb^{*\top}\xb)\cdot \wb^\top\xb/\|\wb\|_q-r ),\\
    g_{S^c}(\wb^*,\wb;\xb) &= \ell'(-\text{sgn}(\wb^{*\top}\xb)\cdot \wb^\top\xb/\|\wb\|_q-r), \\
    g(\wb^*,\wb;\xb)&=g_S(\wb^*,\wb;\xb)+g_{S^c}(\wb^*,\wb;\xb).
\end{align*}
Then the gradient $\nabla_{\wb}L_{\cD}^{p,r}(\wb)$ can be rewritten as
\begin{align*}
\nabla_{\wb}L_{\cD}^{p,r}(\wb)&=\EE\big[g_{S}(\wb^*,\wb;\xb) \text{sgn}(\wb^{*\top}\xb) \nabla_{\wb}h(\wb,\xb+\deb) \ind(S)\big]\notag\\
&\qquad - \EE\big[g_{S^c}(\wb^*,\wb;\xb)\text{sgn}(\wb^{*\top}\xb) \nabla_{\wb}h(\wb,\xb+\deb) \ind(S^c)\big]\notag\\
& = \EE\big[g_{S}(\wb^*,\wb;\xb) \text{sgn}(\wb^{*\top}\xb) \nabla_{\wb}h(\wb,\xb+\deb)\big]\notag\\
& \qquad - \EE\big[g(\wb^*,\wb;\xb) \text{sgn}(\wb^{*\top}\xb) \nabla_{\wb}h(\wb,\xb+\deb) \ind(S^c)\big].
\end{align*}
Then using \eqref{eq:grad.h.formula.main} and the fact that $\deb = -ry\bar\wb/\|\wb\|_q^{q-1}$, it can be shown that
\begin{align}\label{eq:formula_inner_product_main}
\wb^{*\top} \nabla L_{\cD}^{p,r}(\wb) &= \EE\big[g_{S}(\wb^*,\wb;\xb) \text{sgn}(\wb^{*\top}\xb) \wb^{*\top}\nabla_{\wb}h(\wb,\xb+\deb)\big]\notag\\
& \qquad- \EE\big[g(\wb^*,\wb;\xb) \text{sgn}(\wb^{*\top}\xb) \wb^{*\top}\nabla_{\wb}h(\wb,\xb+\deb) \ind(S^c)\big],
\end{align}
where we have defined the quantity $\tilde\wb = \wb^*/\|\wb\|_q - (\bar\wb^\top\wb^*)\wb/\|\wb\|_q^{q+1}$.  This decomposition allows for the label to only play a role through the indicator function $\ind(S^c)$.  

With this notation in order,  we can begin with our proof.  The first step is to show that the key quantity \eqref{eq:formula_inner_product_main} is more negative when $\wb$ is far from $\wb^*$ and when the non-robust classification error of the best robust classifier is small.  

\begin{lemma}\label{lemma:upperbound_opt_grad_product_main}
Let $p\in [2,\infty]$ and $q\in[1,2]$ be such that $1/p+1/q=1$ and $\wb^*=\arg\min_{\|\wb\|_q=1}L_\cD^{p,r}(\wb)$ be the optimal model parameter that achieves minimum $\ell_p$ robust error and $\err_\cD(\wb^*)$ be the clean error achieved by $\wb^*$. Suppose the data distribution $\cD$ satisfies Assumption \ref{assump:isotropic}. For any $\wb$ of $\ell_q$ norm $1$, let $\tilde \wb = \wb^* - (\bar\wb^\top\wb^*)\wb$, $\theta(\wb) = \angle(\wb,\wb^*)$ and $\theta'(\wb)=\angle(-\wb,\tilde\wb)$. 
Let $\sigma = r$.  If $r\le R\|\wb\|_2\sin^{3/2}(\theta'(\wb))/(100U)$, $\err_\cD(\wb^*)\le(2^{14}R^4\|\wb\|)^{-1}U'^2r\sin^2(\theta'(\wb))$ and
\begin{align*}
\sin(\theta(\wb)) \ge  \max\bigg\{\frac{4r}{R\|\wb\|_2},\frac{100 r \sqrt{U/U'}}{R\|\wb\|_2\sin^{1/2}(\theta'(\wb))}\bigg\},
\end{align*}
hold, then it holds that 
\begin{align*}
\wb^{*\top}\nabla L_\cD^{p,r}(\wb)\le -\frac{R^2\|\tilde\wb\|_2\cdot\sin\theta'(\wb)\cdot e^{-1}}{2\|\wb\|_2}
\end{align*}
\end{lemma}

The proof for Lemma \ref{lemma:upperbound_opt_grad_product_main} can be found in Appendix \ref{sec:upperbound_inner_product}.
Lemma \ref{lemma:upperbound_opt_grad_product_main} shows that $\wb^{*\top}\nabla L_\cD^{p,r}(\wb)$ is more negative when the angle $\theta(\wb)$ is large, which intuitively means that the distance $\|\wb-\wb^*\|_2^2$ to the optimal robust classifier will decrease until we reach a point where the angle $\theta(\wb)$ with the optimal robust classifier becomes small (recall the intuition from gradient flow given in \eqref{eq:grad_flow}).  We formalize this into the following lemma, which shows that Algorithm \ref{alg:projected_advtraining} leads to a halfspace that is close to the optimal robust classifier.

\begin{lemma}\label{lemma:convergence_guarantee_psat}
Let $\delta\in(0,1)$ be arbitrary. Then if $r=O(d^{\frac{3}{2p}-\frac{3}{4}})$, set $\eta = O\big(\delta r^3 d^{\frac{1}{2p}-\frac{1}{4}}\big)$ and run Algorithm $\mathsf{PSAT}(p, r)$ for  $K = O\big(d\|\wb_1 - \wb^*\|_2^2\delta^{-2}r^{-4}d^{1/2-1/p}\big)$ iterations, with probability at least $1-\delta$, there exists a $k^*\le K$ such that \begin{align*}
\sin(\theta(\wb_{k^*})) \le \left\{
\begin{array}{ll}
   O\Big(\frac{r d^{\frac{1}{4}-\frac{1}{2p}}}{ 
  \|\wb_{k^*}\|_2^{1/2}}\Big)  & \|\wb_{k^*}\|_2\ge \|\wb^*\|_2,  \\
    O\Big(\frac{r}{\|\wb_{k^*}\|_2}\Big) & \|\wb_{k^*}\|_2<\|\wb^*\|_2.
\end{array}
\right.
\end{align*}
\end{lemma}
The proof for Lemma \ref{lemma:convergence_guarantee_psat} can be found in Appendix \ref{sec:convergence_guarantee_psat}. We can now proceed to complete the proof of Theorem \ref{thm:guarantee_normalized_classifier} based on Lemma \ref{lemma:convergence_guarantee_psat} by showing small $\theta(\wb_{k^*})$ suffices to ensure small robust classification error $\err_\cD^{p,r}(\wb_{k^*})$. The completed proof of Theorem \ref{thm:guarantee_normalized_classifier} can be found in Appendix \ref{sec:proof_theorem_normalized} and we sketch the crucial part as follows.

\begin{proof}[Proof of Theorem \ref{thm:guarantee_normalized_classifier}]
Before characterizing the robust classification error $\err_\cD^{p,r}(\wb_{k^*})$, we first investigate the optimal robust error $\opt_{p,r}=\err_{\cD}^{p,r}(\wb^*)$ and see how it relates to the perturbation radius $r$. In particular,
\begin{align*}
\opt_{p,r} &= \EE_{(\xb,y)\sim\cD}\bigg[\ind\bigg(y \frac{\wb^{*\top}}{\|\wb^*\|_q} (\xb + \deb)\le 0\bigg)\bigg]\notag\\
&=\EE_{(\xb,y)\sim\cD}\big[\ind(y\wb^{*\top}\xb\le r)\big],
\end{align*}
where we use the fact that $\|\wb^*\|_q=1$ in the second equality.
Note that the robust error consists of two disjoint parts of data: (1) the data satisfies  $|\wb^{*\top}\xb|\le r$; and (2) the data satisfies  $|\wb^{*\top}\xb|> r$ and $y\wb^{*\top}\xb<0$. Therefore, we can get lower and upper bounds on  $\OPT_{p,r}$,
\begin{align}\label{eq:bounds_err_w*_main}
\opt_{p,r}&\ge\EE_{\xb\sim\cD_x}\big[\ind(|\wb^{*\top}\xb|\le r)\big]\notag\\
\opt_{p,r}&\le \EE_{\xb\sim\cD_x}\big[\ind(|\wb^{*\top}\xb|\le r)\big] + \err_{\cD}(\wb^*).
\end{align}
By Assumption \ref{assump:isotropic}, we have the data distribution $\cD_x$ satisfies $U$-anti-concentration and $(U',R)$ anti-anti-concentration with $U,R$ being constants. Therefore, it follows that $\EE_{\xb\sim\cD_x}\big[\ind(|\wb^{*\top}\xb|\le r)\big]=\Theta(r\|\wb^*\|_2^{-1})$ since we have $r\|\wb\|_2^{-1}=O(d^{\frac{3}{2p}-\frac{3}{4}}\|\wb\|_2^{-1})\le R$.
Besides, note that we also have $\err_\cD(\wb^*)=O(rd^{2/p-1})\le O(r\|\wb^*\|_2^{-1})$ due to our assumption. Therefore, it is clear that $\OPT_{p,r} = \Theta(r\|\wb^*\|_2^{-1})$.

An argument similar to that used for \eqref{eq:bounds_err_w*_main} leads to the bound \begin{align}\label{eq:bound_robusterr_wk}
\err_{\cD}^{p,r}(\wb_{k^*})&\le\EE_{\xb\sim\cD_x}\big[\ind(|\wb_{k^*}^\top\xb|\le r)\big] + \err_{\cD}(\wb^*)\notag\\
& = O(r\|\wb_{k^*}\|_2^{-1}) + \err_{\cD}(\wb_{k^*}).
\end{align}
We proceed by sketching how we bound each of these two terms.  For $O(r\|\wb_{k^*}\|_2^{-1})$, we only need to characterize the $\ell_2$ norm of $\wb_{k^*}$. In fact by Lemma \ref{lemma:convergence_guarantee_psat}, we can show that under the assumption that $r=O(d^{\frac{3}{2p}-\frac{3}{4}})$ it holds that $\|\wb_{k^*}\|_2 = \Omega(\|\wb^*\|_2)$ (see Appendix \ref{sec:proof_theorem_normalized} for more details), which further implies that $O(r\|\wb_{k^*}\|_2^{-1}) = O(r\|\wb^*\|_2^{-1}) = O(\opt_{p,r})$. 

The next step is to characterize $\err_{\cD}(\wb_{k^*})$.  We can do so by comparing it with the classification error of $\wb^*$,
\begin{align}\label{eq:bound_err_wk_intermediate}
\err_{\cD}(\wb_{k^*})&\le  \err_{\cD}(\wb^*) + |\err_{\cD}(\wb_{k^*})-\err_{\cD}(\wb^*)|\notag\\
&\le 2\err_{\cD}(\wb^*) + \EE_{\xb\sim\cD_x}[\ind(\wb_{k^*}^\top\xb\neq \wb^{*\top}\xb)] \notag\\
&\le O(r\|\wb^*\|_2^{-1}) + \Theta(\theta(\wb_{k^*})),
\end{align}
where the last inequality is due to the fact that $\cD_x$ is isotropic  (see Appendix \ref{sec:proof_theorem_normalized} for more details). Then by Lemma \ref{lemma:convergence_guarantee_psat} it is clear that 
\begin{align}\label{eq:bound_theta}
\theta(\wb_{k^*}) = O\bigg(\frac{r d^{\frac{1}{4}-\frac{1}{2p}}}{ 
  \|\wb_{k^*}\|_2^{1/2}}\bigg) =  O\bigg(\frac{r d^{\frac{1}{4}-\frac{1}{2p}}}{ 
  \|\wb^*\|_2^{1/2}}\bigg)
\end{align}
since we have shown that $\|\wb_{k^*}\|_2=\Omega(\|\wb^*\|_2)$.
Consequently, combining \eqref{eq:bound_theta} and \eqref{eq:bound_err_wk_intermediate} and further substituting into \eqref{eq:bound_robusterr_wk}, we get that $\err_{\cD}^{p,r}(\wb_{k^*})$ is at most
\begin{align*}
 O\bigg(\frac{r d^{\frac{1}{4}-\frac{1}{2p}}}{ 
  \|\wb^*\|_2^{1/2}}\bigg)  &= O\big(d^{\frac{1}{4}-\frac{1}{2p}}\cdot\|\wb^*\|_2^{1/2}\cdot\OPT_{p,r}\big)
\end{align*}
since $\opt_{p,r} = \Theta(r\|\wb^*\|_2^{-1})$. This completes the proof.

\end{proof}




\section{Conclusion and Future Work}\label{sec:conclusion}
In this work we analyzed the properties of adversarial training for learning halfspaces with noise.  We provided the first guarantee that adversarial training provably leads to robust classifiers when the data distribution has label noise.  In particular, we established that adversarial training leads to approximate minimizers for the robust classification error under $\ell_p$ perturbations for many distributions.  For typical convex loss functions like the cross entropy or hinge loss, we showed that adversarial training can achieve robust classification error $\tilde O\big(\sqrt{\opt_{2,r}}\big)$
when $p=2$ and  $\tilde O\big(d^{1/4}\sqrt{\opt_{\infty,r}} + d^{1/2} \opt_{\infty, r} \big)$ for $\ell_\infty$ when $p=\infty$ for distributions satisfying anti-concentration properties.  We showed that the robust classification error guarantees can be improved if we instead use a nonconvex sigmoidal loss, with guarantees of $O(\opt_{2,r})$ for $p=2$ and $O(d^{1/4}\opt_{\infty,r})$ for $p=\infty$ in the worst case.  For future work, we are keen on understanding whether or not adversarial training provably leads to robust classifiers for more complicated function classes than halfspaces.

\appendix

\section{Proofs for Typical Convex Losses}

\subsection{Proof of Lemma \ref{lemma:convergence_advtraining}}\label{appendix:cross.ent.empirical.risk}
\begin{proof}[Proof of Lemma \ref{lemma:convergence_advtraining}]
Throughout this proof we assume $\ell$ is convex and $L$-Lipschitz.  Following the notation of Algorithm \ref{alg:adv_training}, denote 
\[ \dit := \argmax_{\norm{\deb}_p\leq r} \ell(y_i \wb_k^\top(\xb_i + \deb)).\]
Note that $\dit = \dit(\wb_k, y_i, \xb_i)$ depends on $\wb_k$.  To analyze the convergence of gradient descent on the robust risk, we introduce a reference vector $\wb^*\in \RR^d$, and consider the decomposition
\begin{align*}
    \|\wb_k - \wb^*\|_2^2 - \|\wb_{k+1} - \wb^*\|_2^2  &= 2 \eta \left \langle \frac 1 n \summ i n \ell'\big(y_i \wb_k^\top (\xb_i + \dit)\big) y_i (\xb_k + \dit), \wb_k - \wb^*\right \rangle \\
    &-\eta^2 \Big\| \frac 1 n \summ i n \ell'\big(y_i \wb^\top(\xb_i + \dit) \big) y_i (\xb_i + \dit) \Big\|_2^2.
\end{align*}
For the first term, note that for every $(\xb_i, y_i)\in S$ and $k\in \mathbb N$,
\begin{align}\nonumber
    &\ell'\big(y_i \wb_k^\top(\xb_i + \dit)\big)( \wb_k^\top(\xb_i + \dit) - \wb^{*\top} (\xb_i + \dit))\\ \nonumber
    & \geq \ell\big(y_i \wb_k^\top(\xb_i + \dit)\big) - \ell\big(y_i \wb^{*\top} (\xb_i + \dit) \big) \\ \nonumber
    &\geq \ell\big(y_i \wb_k^\top(\xb_i + \dit)\big) - \sup_{\norm{\deb}_p \leq r} \ell\big(y_i \wb^{*\top} (\xb_i + \deb) \big),
\end{align}
where the first line follows by convexity of $\ell$.  This allows for us to bound
\begin{align}\nonumber
    &\f 1 n \summ i n \ell'\big(y_i \wb_k^\top (\xb_i + \dit)\big) y_i  (\wb_k - \wb^*)^\top(\xb_k + \dit) \\ \nonumber
    &\geq \frac 1 n \summ i n \Bigg[  \ell\big(y_i \wb_k^\top(\xb_i + \dit)\big) - \sup_{\norm{\deb} \leq r} \ell\big(y_i \wb^{*\top} (\xb_i + \deb) \big) \Bigg] \\
    &= \lsrob(\wb_k) - \lsrob(\wb^*).\label{eq:adv.training.ip.lb_gd}
\end{align}
For the gradient upper bound, under the assumption that $\ell$ is $L$-Lipschitz,
\begin{align*}
    \Big\| \frac 1 n \summ i n \ell'\big(y_i \wb^\top(\xb_i + \dit) \big) y_i (\xb_i + \dit) \Big\|_2^2 &\leq  \frac 1 n \summ i n \| \ell'\big(y_i \wb^\top(\xb_i + \dit) \big) y_i (\xb_i + \dit) \|_2^2 \\
    &\leq \frac 1 n \summ i n L^2 \norm{\xb_i + \dit}_2^2 \\
    &\leq 2 L^2 \frac 1 n \summ i n (\norm{\xb_i}_2^2 + \norm{\dit}_2^2)\\
    &\leq 2 L^2 \sup_{\xb \sim \calD_x } \norm{\xb}_2^2 + 2 L^2 \sup_{\norm{\deb}_p \leq r} \norm{\deb}_2^2 \\
    &\leq 2 L^2 \sup_{\xb \sim \calD_x} \norm{\xb}_p^2 \cdot \frac{\norm{\xb}_2^2}{\norm{\xb}_p^2} + 2 L^2 \sup_{\norm{\deb}_p \leq r} \norm{\deb}_p^2 \cdot \frac{\norm{\deb}_2^2}{\norm{\deb}_p^2}\\
    &\leq \begin{cases} 2 L^2(1 + r), & p \leq 2,\\
    2L^2(d + rd),& p > 2.\end{cases}
\end{align*}
In the first inequality, we use Jensen's inequality.  In the second we use that $\ell$ is $L$-Lipschitz.  The third inequality follows by Young's inequality.  In the last, we use that $\norm{x}_p\leq 1$ and that $p \mapsto \norm{\xb}_p$ is a decreasing function for fixed $\xb$, together with the bound $\norm{\xb}_2/\norm{\xb}_\infty \leq \sqrt{d}$.   Assuming without loss of generality that $r \leq 1$, this shows that 
\begin{equation}\label{eq:adv.training.norm.bound_gd}
    \Big\| \frac 1 n \summ i n \ell'\big(y_i \wb^\top(\xb_i + \dit) \big) y_i (\xb_i + \dit) \Big\|_2^2 \leq H := \begin{cases} 4 L^2, & p \leq 2,\\
    4L^2d,& p > 2.\end{cases}
\end{equation}
Putting \eqref{eq:adv.training.ip.lb_gd} and \eqref{eq:adv.training.norm.bound_gd} together, we have for $\eta \leq \eps H/4$,
\begin{align}\label{eq:onestep_norm_decrease_gd}
    \|\wb_k - \wb^*\|_2^2 - \|\wb_{k+1} - \wb^*\|_2^2 \geq 2 \eta(\lsrob(\wb_k) - \lsrob(\wb^*)) - \eta^2 H \geq 2 \eta(\lsrob(\wb_k) - \lsrob(\wb^*) - \eps/2).
\end{align}
We can use the above to bound the number of iterations until we reach a point with $\lsrob(\wb_k) \leq \lsrob(\wb^*) + \eps$.  Let $K$ be the number of iterations until we reach such a point, so that for $k=1, \dots, K$, it holds that $\lsrob(\wb_k) > \lsrob(\wb^*) + \eps$.  Then \eqref{eq:onestep_norm_decrease_gd} implies that for each of $k=1,\dots, K$,
\[ \|\wb_k - \wb^*\|_2^2 - \|\wb_{k+1} - \wb^*\|_2^2 \geq \eta \eps.\]
In particular, at every such iteration, $\| \wb_k - \wb^*\|_2^2$ decreases by at least $\eta \eps$.  There can only be $\|\wb_0-\wb^*\|_2^2 / (\eta \eps)$ such iterations.  This shows that there exists some $k^*\leq K = \|\wb_0-\wb^*\|_2^2 \eta^{-1} \eps^{-1}$ for which $\lsrob(\wb_k) \leq \lsrob(\wb^*)+\eps$, and this $k^*$ satisfies $\|\wb_{k^*}-\wb^*\|_2\leq \|\wb_0 - \wb^*\|_2$.  

\end{proof}

\subsection{Proof of Lemma \ref{lemma:generalization_GD}}\label{appendix:cross.ent.generalization}
\begin{proof}[Proof of Lemma \ref{lemma:generalization_GD}]
We will follow the proof of Theorem 2 in \citet{yin2019rademacher}. Note that we have $\|\wb_{k^*}-\wb^*\|_2\le \|\wb_{k^*}\|_2$. Therefore, we have $\|\wb_{k^*}\|_2\le 2\|\wb^*\|_2$. Therefore we have
\begin{align*}
\|\wb_{k^*}\|_q \le \|\wb_{k^*}\|_2\cdot \max\{1, d^{1/q-1/2}\}\le 2\|\wb^*\|_2\cdot \max\{1, d^{1/q-1/2}\}\le 2\|\wb^*\|_q d^{|1/q-1/2|}\le 2d^{|1/q-1/2|}\rho.
\end{align*}
Then let $\rho' =  2d^{|1/q-1/2|}\rho$, we define the following function class $\cF\subseteq \RR^{\cX\times\{\pm 1\}}$,
\begin{align*}
\cF:=\bigg\{\min_{\xb'\in \cB_{p}(\xb,r)}y\la \wb, \xb'\ra: \|\wb\|_q\le \rho \bigg\}=\big\{y\la \wb, \xb\ra - r\|\wb\|_q: \|\wb\|_q\le \rho' \big\}.
\end{align*}
Since $\ell$ is decreasing and 1-Lipschitz, $\ell(y f(\xb)) \leq \ell(0) + \|\wb\|_q\|\xb\|_p + r \|\wb\|_q \leq \ell(0) + (1+r) \rho' :=B$ holds for any $f\in \mathcal{F}$.  Thus, by \citet[Corollary 1]{yin2019rademacher}, we know that with probability at least $1-\delta$, for arbitrary $k^*<K$ it holds that
\begin{align*}
\ldrob(\wb_{k^*})&\le \lsrob(\wb_{k^*})+2B\mathfrak{R}(\cF) + 3B\sqrt{\frac{\log(2/\delta)}{2n}}\notag.
\end{align*}
We now want to apply Lemma \ref{lemma:convergence_advtraining}.  Note that the iteration complexity $K$ depends on the $\ell_2$ norm of $\wb^*$.  Using H\"older's inequality, $\|\wb^*\|_2 \leq \max\{1, d^{\frac 1q - \frac 12}\}\norm{\wb}_q \leq \max\{1, d^{\frac 1q - \frac 12}\}\rho$.  Thus, by taking $K = \eta^{-1} \eps^{-1} \max\{1, d^{\frac 2q - 1}\}\rho^2$, we have the inequality
\begin{align*}
\lsrob(\wb_k^*) &\le L_{\cS}^{p,r}(\wb^*)+\epsilon+2B\mathfrak{R}(\cF) + 3B\sqrt{\frac{\log(2/\delta)}{2n}}.
\end{align*}
Applying~\citet[Corollary 1]{yin2019rademacher} once more to $\wb^*\in \mathcal{F}$, we get
\begin{equation}\label{eq:rademacher_generalization_fix}
L_{\cD}^{p,r}(\wb_{k^*})\le L_{\cD}^{p,r}(\wb^*)+\epsilon+4B\mathfrak{R}(\cF) + 6B\sqrt{\frac{\log(2/\delta)}{2n}}
\end{equation}
Moreover, applying union bound for all possible $k^*< K$, we can get with probability at least $1-\delta$, 
\begin{equation}
    \label{eq:rademacher_generalization}
L_{\cD}^{p,r}(\wb_{k^*})\le L_{\cD}^{p,r}(\wb^*)+\epsilon+4B\mathfrak{R}(\cF) + 6B\sqrt{\frac{\log(2K/\delta)}{2n}}
\end{equation}
Therefore, then rest effort will be made to prove the upper bound of the Rademacher complexity. 
Based on the definition of Rademacher complexity, we have
\begin{align*}
\mathfrak{R}(\cF) &= \frac{1}{n}\EE_{\sigma}\bigg[\sup_{\|\wb\|_q\le \rho' }\sum_{i=1}^n\sigma_i\big(y_i\wb^\top\xb_i - r\|\wb\|_q\big)\bigg]\notag\\
&= \frac{1}{n}\EE_{\sigma}\bigg[\sup_{\|\wb\|_q\le \rho' }\wb^\top \ub - v\|\wb\|_q\bigg],
\end{align*}
where $\sigma_i$ is i.i.d. Rademacher random variable, $\ub = \sum_{i=1}^n\sigma_iy_i\xb_i$ and $v = r\sum_{i=1}^n\sigma_i$. Then we have
\begin{align*}
\sup_{\|\wb\|_q\le \rho' }\wb^\top \ub - v\|\wb\|_q \le \sup_{\|\wb\|_q\le \rho'} \|\wb\|_q(\|\ub\|_p - v)\le \rho' |\|\ub\|_p - v|.
\end{align*}
Therefore, we have
\begin{align*}
\mathfrak{R}(\cF)&\le \frac{\rho' }{n}\EE_{\sigma}\big[|\|\ub\|_p-v|\big]\notag\\
&\le \frac{\rho'}{n}\Big[\EE_{\sigma}[\|\ub\|_p]+\EE_{\sigma}[|v|]\Big]\notag\\
&=\frac{\rho'}{n}\bigg[\EE_{\sigma}\bigg[\bigg\|\sum_{i=1}^n\sigma_iy_i\xb_i\bigg\|_p\bigg]+ \frac{r \rho'}{n} \EE_{\sigma}\bigg[\bigg|\sum_{i=1}^n\sigma_i\bigg|\bigg]\bigg]\notag\\
&=\frac{\rho'}{n}\EE_{\sigma}\bigg[\bigg\|\sum_{i=1}^n\sigma_i\xb_i\bigg\|_p\bigg]+\frac{\rho' r}{\sqrt{n}}\notag\\
&: = \rho ' \mathfrak{R}_p + \frac{\rho' r}{\sqrt{n}}.
\end{align*}
Plugging the above inequality into \eqref{eq:rademacher_generalization} completes the proof.
\end{proof}

\subsection{Proofs for Example \ref{example:soft.margin.l2} and Example \ref{example:soft.margin.lq}} \label{appendix:soft.margin}
We first show the properties given in Example \ref{example:soft.margin.l2}.  Part 1 follows by~\citet{frei2020halfspace}.  For Part 2, that $\phi_{\bar \vb,2}(\gamma) = O(\gamma)$ follows by~\citet{frei2020halfspace}.  For Part 3, to show that $(U', R)$-anti-anti-concentration implies $\phi_{\bar \vb,2}(\gamma) = \Omega(\gamma)$, we first note that if $\calD_x$ satisfies $(U', R)$ anti-anti-concentration defined in terms of projections onto two dimensional subspaces, then it also satisfies $(U', \Theta(R))$ anti-anti-concentration onto projections defined in terms of projections onto one dimensional subspaces, since we have the set of inclusions
\[ \{ \zb = (z_1,z_2):\ |z_i|\leq R/2,\ i=1,2 \} \subset \{ \zb:\ \|\zb\|\leq R \} \subset \{ \zb = (z_1,z_2):\ |z_i|\leq R,\ i=1,2 \}.\]
Therefore, denoting $p_{\vb}(\cdot)$ as the marginal density of $\calD_x$ onto the subspace spanned by $\vb \in \RR^d$, we have for $\gamma \leq \Theta(R)$,
\begin{align*}
\phi_{\bar \vb,2}(\gamma) &= \PP(|\bar \vb^\top \xb| \leq \gamma) \\
&= \int p_{\bar \vb}(z_1) \ind(|z_1| \leq \gamma) \dd z_1\\
&\geq \frac{2\gamma }{U'}.
\end{align*}
This shows that $\phi_{\bar \vb,2}(\gamma) = \Omega(\gamma)$ when $\gamma \leq \Theta(R)$.  
Finally, Part 4 of Example \ref{example:soft.margin.l2} follows by~\citet[Theorem 11]{balcan2017logconcave}, using the fact that the marginals of any log-concave distribution are again log-concave~\citep[Theorem 5.1]{lovasz2007logconcave}.

We now show the properties of Example \ref{example:soft.margin.lq}.
For the first part, the general case of $q\in (1,2]$ follows by H\"older's inequality, since we can write
\begin{align*}
    \| \vb\|_q = \left(\summ j q |v_j|^q\right)^{1/q} \leq \|(1, \cdots,1)\|_{2/(2-q)}^{1/q} \cdot \| (|v_1|^q, \dots, |v_d|^q)\|_{2/q}^{1/q} = d^{\frac 1 q - \frac 1 2} \|\vb\|_2.
\end{align*}
For the second, we use Example \ref{example:soft.margin.l2} and \eqref{eq:soft.margin.lq.intermediate} to get that 
\[ \phi_{\bar \vb, q}(\gamma) = \phi_{\bar \vb/\norm{\bar \vb}} (\gamma / \|\bar \vb\|_2) = \Omega(\gamma / \norm{\bar \vb}_2) = \Omega(\gamma).\]
In the last equality, we have used that $\|\vb\|_2\leq \|\vb \|_q$.

\subsection{Proof of Lemma \ref{lemma:adversarial.training.robust.risk.surrogate}}\label{appendix:surrogate.vs.classification.robust}
\begin{proof}[Proof of Lemma \ref{lemma:adversarial.training.robust.risk.surrogate}]
We use an argument similar to that used by~\citet{frei2020halfspace} for learning halfspaces with noise.  We write the surrogate risk as a sum of three terms,
\begin{align*}
\nonumber
    \ldrob(\vb) &= \E_{(\xb,y)} \left[\sup_{\xb'\in \bpxr} \ell(y \vb^\top \xb')\right] \\
    &\leq \E\left[\sup_{\xb'\in \bpxr} \ell(y \vb^\top \xb') \ind\left(y \bar \vb^\top \xb'\leq  0\right)\right] + \E\left[\sup_{\xb'\in \bpxr} \ell(y \vb^\top \xb') \ind\left(0 < y \bar \vb^\top \xb' \leq \gamma\right)\right] \\
    &\quad + \E\left[\sup_{\xb'\in \bpxr} \ell(y \vb^\top \xb') \ind\left (y \bar \vb^\top \xb'  > \gamma\right)\right] .\label{eq:objective.decomposition.3terms}
\end{align*}
For the first term, we use that $\ell$ is $L$-Lipschitz and decreasing together with H\"older's inequality to get
\begin{align*}
\nonumber
    \E\left[\sup_{\xb'\in \bpxr} \ell(y \vb^\top \xb') \ind(y \bar \vb^\top \xb' \leq 0)\right] &\leq \E\left [\sup_{\xb'\in \bpxr} (\ell(0) + L |\vb^\top \xb' | ) \ind(y \bar \vb^\top \xb' \leq 0)\right ] \\
    &\leq \E\left [\sup_{\xb'\in \bpxr} (\ell(0) + L \norm{\vb}_q \norm{\xb'}_p ) \ind(y \bar \vb^\top \xb' \leq 0)\right ]
    \nonumber\\
    &\leq (\ell(0) + L \rho) \E\left [\sup_{\xb'\in \bpxr} \ind(y \bar \vb^\top \xb' \leq 0)\right ]\\
    &= (\ell(0) + L \rho) \opt_{p,r}.
\end{align*}
In the last inequality we use that $\norm {\xb}_p \leq 1$ a.s.  

For the second term, we first notice that since $|\sip {\bar \vb}{\xb}| \leq |\sip{\bar \vb}{\xb-\xb'}| + |\sip{\bar \vb}{\xb'}|$, we have the inclusion
\begin{equation*}
    \{ |\sip{\bar \vb}{\xb'}| \in (0, \gamma] \} \subset \{ |\sip{\bar \vb}{\xb}| \leq |\sip{\bar \vb}{\xb-\xb'}| + \gamma \}.
\end{equation*}
Therefore, we can bound
\begin{align*} 
\label{eq:surrogate.ub.by.zeroone.secondterm}
   \E\left[\sup_{\xb'\in \bpxr} \ell(y \vb^\top \xb') \ind\left(0 < y \bar \vb^\top \xb' \leq \gamma\right)\right] &\leq \ell(0) \E \left[ \sup_{\xb'\in \bpxr}  \ind\left(0 < y \bar \vb^\top \xb' \leq \gamma\right)\right]  \\
   &\leq \ell(0) \E\left[ \sup_{\xb'\in \bpxr} \ind(|\sip{\bar \vb}{\xb}| \leq |\sip{\bar \vb}{\xb-\xb'}| + \gamma ) \right] \\
   &\leq \ell(0) \E\left[ \sup_{\xb'\in \bpxr} \ind(|\sip{\bar \vb}{\xb}| \leq \norm{\bar \vb}_q \norm{\xb-\xb'}_p + \gamma ) \right]\\
   &\leq \ell(0) \E\left[ \sup_{\xb'\in \bpxr} \ind(|\sip{\bar \vb}{\xb}| \leq r + \gamma ) \right]\\
   &\leq \ell(0) \E\left[ \ind(|\sip{\bar \vb}{\xb}| \leq r + \gamma ) \right]\\
   &\leq \ell(0) \phi_{\bar \vb,q} ( r + \gamma).
\end{align*}
where we have used that $\ell$ is decreasing in the first inequality and Definition \ref{def:softmargin} in the second. 
Finally, for the last term, we can use that $\ell$ is decreasing to get
\begin{align*}
\E\left[\sup_{\xb'\in \bpxr} \ell(y \vb^\top \xb') \ind\left (y \bar \vb^\top \xb'  > \gamma\right)\right] &= \E\left[\sup_{\xb'\in \bpxr} \ell(y \rho \bar \vb^\top \xb') \ind\left(y \rho\bar \vb^\top \xb' >  \rho\gamma \right)\right] \\
&\leq \E\left[\sup_{\xb'\in \bpxr} \ell(\rho \gamma) \right]\\
&=\ell( \rho \gamma).   
\end{align*}
To see the final claim, note that we can write the event defining $\poproberr(\vb)$ as
\begin{align*}
  \Bigg\{ \sup_{\xb'\in \bpxr} \sign(\sip{\vb}{\xb}) \neq y)\Bigg \} &=\Bigg \{ \sup_{\xb'\in \bpxr} y \sip{\vb}{\xb'} < 0 \Bigg\} \\
  &=\Bigg \{ \sup_{\xb'\in \bpxr} \ell (  y \sip{\vb}{\xb'}) > \ell(0) \Bigg\}.
\end{align*}
The final equality follows since $\ell$ is decreasing.  This proves \eqref{eq:surrogate.vs.classification.key}.  For the final claim of the Lemma, by Markov's inequality,
\begin{align*}
    \poproberr(\vb) &= \PP_{(\xb, y)} \Big( \sup_{\xb'\in \bpxr} y \sip{\vb}{\xb'} < 0\Big) \\
    &= \PP_{(\xb,y)} \Big(  \sup_{\xb'\in \bpxr} \ell (  y \sip{\vb}{\xb'}) > \ell(0) \Big)\\
    &\leq [\ell(0)]^{-1} \E_{(\xb,y)}\Big[  \sup_{\xb'\in \bpxr} \ell(y \sip{\vb}{\xb'}) \Big] \\
    &= [\ell(0)]^{-1} \ldrob(\vb).
\end{align*}
\end{proof}

\subsection{Proofs of Corollaries}\label{appendix:corollaries}
\begin{proof}[Proof of Corollary \ref{corollary:hard.margin}]
First, let us note that $\ell(0) = \log 2 \approx 0.693$ and $\ell^{-1}(1/\eps) \in [\log(1/(2\eps)), \log(2/\eps)]$ for the cross entropy loss.  Additionally, by standard arguments from Rademacher complexity (see, e.g.,~\citet{kakade2009complexitylinear}, Theorem 1),
\begin{equation}\label{eq:robust.rademacher.ub}
\mathfrak{R}_p = \frac 1 n \EE_{\sigma_i \iid \mathrm{Unif}(\pm 1)} \left\|\summ i n \sigma_i \xb_i\right\|_p = \begin{cases} O\left(\frac{p}{\sqrt n}\right), & p \in[2,\infty),\\ O\left(\frac{\log d}{\sqrt n}\right) = \tilde O\left(\frac{1}{\sqrt n}\right) , & p =\infty,\end{cases}.
\end{equation}
By the definition of hard margin, the soft margin function satisfies $\phi_{\bar \vb, q}(\gamma) = 0$ for $\gamma<\gamma_0$.  Thus applying Theorem \ref{thm:crossentropy} with $r = (1-\nu )\gamma_0$ and $\gamma = \nu \gamma_0$, we get
\begin{equation} 
\erobrp (\wb_k^*)\leq O ( \log(1/\eps) \gamma_0^{-1} \nu^{-1} \opt_{p,r}) + \tilde O\left(\frac{ \gamma_0^{-1} \nu^{-1} B \log(1/\eps) \sqrt{\log(1/\delta)}}{\sqrt n}\right) + \eps,
\end{equation}
where the $\tilde O(\cdot)$ in the second term hides the possible logarithmic dependence on $d$ when $p=\infty$.  
Now set $\eps = \opt_{p,r}$ and let $n = \tilde \Omega(\gamma_0^{-2}\nu^{-2} B^2 \sqrt{\log(1/\delta)} \log^2(1/\opt_{p,r}) \opt_{p,r}^{-2})$.  Since $B = \ell(0) + (1+r)\gamma^{-1} \ell^{-1}(\eps) = \tilde O(\gamma^{-1})$, this completes the proof. 
\end{proof}

\begin{proof}[Proof of Corollary \ref{corollary:anti.concentration}]
Denote $q$ as the H\"older conjugate to $p$, so $1/p+1/q=1$.  Since the inclusion $\{|\sip{\wb^*}{\xb}| \leq r \} \subset \{ y \sip{\wb}{\xb} \leq r \}$ holds, we have
\[ \phi_{\wb, q}(r) \leq \PP(y\sip{\wb}{\xb} \leq r) = \poproberr(\wb).\]
Therefore $\phi_{\wb, q}(r)\leq \poproberr(\wb)$.  If $\calD_x$ satisfies $(U', R)$-anti-anti-concentration, then $\phi_{\wb,q}(r) = \Omega(r)$ for any $q \in [1,2]$ by Example \ref{example:soft.margin.lq} since $r\leq R$.  This shows that $r = O(\poproberr(\wb))$ holds for any $\wb$ satisfying $\|\wb\|_q=1$, and hence $r = O(\opt_{p,r})$ holds.  When $q$ is the H\"older conjugate to $p$, $q$ satisfies $1/q = 1-1/p.$  Thus, by Example \ref{example:soft.margin.lq},
\begin{equation} \label{eq:soft.margin.lq.bound}
\phi_{\bar \vb, q}(\gamma+r) \leq O\big((\gamma + r) d^{\frac 12 - \frac 1p}\big) \leq O(\gamma d^{\frac 12 - \frac 1p}) + O(d^{\frac 12 - \frac 1p} \opt_{p,r}).
\end{equation}
  By Theorem \ref{thm:crossentropy}, the robust classification error $\erobrp (\wb_{k^*})$ for weights $\wb_{k^*}$ found by Algorithm \ref{alg:adv_training} is at most 
\[ O ( \log(1/\eps) \gamma^{-1} \opt_{p,r}) + \phi_{\bar \vb, q}(r + \gamma) + O\left(\frac{ \gamma^{-1} B \log(1/\eps) \sqrt{\log(1/\delta)}}{\sqrt n}\right) + O(B \gamma^{-1} \log(1/\eps) \mathfrak{R}_p) + \eps.\]
Let now $\eps = \opt_{p,r}$ and choose $\gamma = \opt_{p,r}^{1/2} d^{\frac{1}{2p}-\frac{1}4}$ so that $\gamma^{-1} \opt_{p,r}$ and $\gamma d^{\frac 12 - \frac 1p}$ (coming from \eqref{eq:soft.margin.lq.bound}) are of the same order.  This results in 
\begin{equation}\label{eq:rob.err.ub.opt.plus.r}
    \erobrp(\wb_{k^*}) \leq O(d^{\frac 14 - \frac{1}{2p}} \opt_{p,r}^{1/2} \log(1/\opt_{p,r})) + O(d^{\frac 12 - \frac 1p} \opt_{p,r}) + \tilde O\left( \frac{ d^{\frac 14 - \frac 1{2p}} Br \opt_{p,r}^{-1/2} \log(1/\opt) }{\sqrt n}\right).
\end{equation}
Taking $n = \tilde \Omega(d^{\frac 12 - \frac{1}{p}} B^2 r^2 \opt_{p,r}^{-1})$, and using the fact that $B = O(r \gamma^{-1} \ell^{-1}(\eps)) = \tilde O(d^{\frac 14 - \frac 1{2p}} \opt_{p,r}^{-1/2})$ completes the proof. 
\end{proof}

\section{Guarantees of SGD for Convex Losses}\label{sec:SGD_guarantee}
Our results for SGD will rely upon the assumption that $\ell(z)$ is $M$-smooth, i.e. $\ell''(z) \leq M$.   This allows for us to deal with unbounded Lipschitz activations and also get high-probability guarantees for the output of SGD. In particular, we derive the robust guarantee of the output of SGD in the following theorem.

\begin{theorem}\label{thm:guarantee_SGD}
Assume $\ell$ is convex, decreasing, $L$-Lipschitz, and  $M$-smooth. Let $p\in [1,\infty]$, and assume that $\norm{\xb}_p\leq 1$ a.s.  If $p\leq 2$, let $H = 4L^2$, and if $p > 2$, let $H = 4L^2d$.  Let $\eps>0$ and be arbitrary.  If $\eta \leq \eps H^{-1}/4$, then for any initialization $\wb_1$, if we denote $\wb_k$ as the $k$-th iterate of online-SGD based adversarial training and $\cS=\{\xb_k\}_{k=1,\dots,K}$ be all observed data, by taking $K = 2\eps^{-1}\eta^{-1} \norm{\wb_1 - \wb^*}_2^2$, we have,
\begin{align*}
\min_{k\le K}\err_{\cD}^{p,r}(\wb_k) \le  [\ell'(0)]^{-2}\cdot\bigg(32M L_\cD^{p,r}(\wb^*) + 16M\epsilon + \frac{\big[16M\big(\ell(0)+L\rho(1+r)\big)+4L^2\big]\cdot \log(2/\delta)}{K}\bigg).
\end{align*}
\end{theorem}
For simplicity we treat $M$, $\ell(0)$, $[\ell'(0)]^{-1}$, and $L$ as constants. Then we can set $\epsilon = O(1/K)$ and Theorem \ref{thm:guarantee_SGD} implies that
\begin{align*}
\min_{k\le K}\err_{\cD}^{p,r}(\wb_k) = O(L_{\cD}^{p,r}(\wb^*)) + \tilde O(\rho/K).
\end{align*}
Further note that $\wb^*$ is arbitrary. Then similar to Lemma \ref{lemma:adversarial.training.robust.risk.surrogate}, let $\bar\vb = \min_{\|\wb\|_q=1}\err_\cD^{p,r}(\wb)$, we can define $\wb^* = \rho\bar\vb$, which satisfies
\begin{align*}
L_{\cD}^{p,r}(\wb^*))\le  (\ell(0) + \rho ) \opt_{p,r} + \ell(0) \phi_{\bar \vb,q}(r + \gamma) + \ell(\rho \gamma)
\end{align*}
for arbitrary $\gamma\ge 0$. Then similar to Theorem \ref{thm:crossentropy}, set $\rho = \gamma^{-1}\ell^{-1}(1/\epsilon)$, we get
\begin{align*}
\min_{k\le K}\err_{\cD}^{p,r}(\wb_k) \le C\cdot \big[\big(\ell(0) + \gamma^{-1}\ell^{-1}(1/\epsilon) \big)\cdot \opt_{p,r} + \ell(0) \phi_{\bar \vb,q}(r + \gamma) + \epsilon\big] + \tilde O\big(\gamma^{-1}\ell^{-1}(1/\epsilon)/K\big)
\end{align*}
for some absolute constant $C$. Then it can be seen that the first term (in the first bracket) is nearly identical to the first three terms of the bound in Theorem \ref{thm:crossentropy} up to some constant factors. Then Corollaries \ref{corollary:hard.margin} and \ref{corollary:anti.concentration} also hold, implying that the output of online-SGD adversarial training enjoys the same robustness guarantee as that of full-batch gradient descent adversarial training.

\subsection{Proof of Theorem \ref{thm:guarantee_SGD}}
We first provide the convergence guarantee of SGD in the following lemma.
\begin{lemma}\label{lemma:convergence_advtraining_SGD}
Assume $\ell$ is convex, decreasing, and $L$-Lipschitz.  Let $\wb^*\in \RR^d$ be arbitrary.  Let $p\in [1,\infty]$, and assume that $\norm{\xb}_p\leq 1$ a.s.  If $p\leq 2$, let $H = 4L^2$, and if $p > 2$, let $H = 4L^2d$.  Let $\eps>0$ and be arbitrary.  If $\eta \leq \eps H^{-1}/4$, then for any initialization $\wb_1$, if we denote $\wb_k$ as the $k$-th iterate of Algorithm \ref{alg:adv_training} and $\cS=\{\xb_k\}_{k=1,\dots,K}$ be all observed data, by taking $K = 2\eps^{-1}\eta^{-1} \norm{\wb_1 - \wb^*}_2^2$, we have
\begin{align*}
\frac{1}{K}\sum_{k=1}^{K}L_k^{p,r}(\wb_k)\le L_{\cS}^{p,r}(\wb^*) +\epsilon.
\end{align*}
\end{lemma}

\begin{proof}
Following the notation of Algorithm \ref{alg:adv_training}, denote 
\[ \deb_k := \argmax_{\norm{\deb}_p\leq r} \ell(y_k \wb_k^\top(\xb_k + \deb)).\]
Note that $\deb_k = \deb_k(\wb_k, y_k, \xb_k)$ depends on $\wb_k$.  To analyze the convergence of gradient descent on the robust risk, we introduce a reference vector $\wb^*\in \RR^d$, and consider the decomposition
\begin{align*}
    \|\wb_k - \wb^*\|_2^2 - \|\wb_{k+1} - \wb^*\|_2^2  &= 2 \eta \langle  \ell'\big(y_k \wb_k^\top (\xb_k + \deb_k)\big) y_k (\xb_k + \deb_k), \wb_k - \wb^*\rangle \\
    &-\eta^2 \big\| \ell'\big(y_k \wb^\top(\xb_k + \deb_k) \big) y_k (\xb_k + \deb_k) \big\|_2^2.
\end{align*}
For the first term, note that for every $(\xb_k, y_k)\in S$ and $k\in \mathbb N$,
\begin{align*}\nonumber
    &\ell'\big(y_k \wb_k^\top(\xb_k + \deb_k)\big)( \wb_k^\top(\xb_k + \deb_k) - \wb^{*\top} (\xb_k + \deb_k))\\
    & \geq \ell\big(y_k \wb_k^\top(\xb_k + \deb_k)\big) - \ell\big(y_k \wb^{*\top} (\xb_k + \deb_k) \big) \\
    &\geq \ell\big(y_k \wb_k^\top(\xb_k + \deb_k)\big) - \sup_{\norm{\deb}_p \leq r} \ell\big(y_k \wb^{*\top} (\xb_k + \deb) \big),
\end{align*}
where the first line follows by convexity of $\ell$.  This allows for us to bound
\begin{align}\nonumber
    &\ell'\big(y_k \wb_k^\top (\xb_k + \deb_k)\big) y_k  (\wb_k - \wb^*)^\top(\xb_k + \deb_k) \\
    &\geq  \ell\big(y_k \wb_k^\top(\xb_k + \deb_k)\big) - \sup_{\norm{\deb} \leq r} \ell\big(y_k \wb^{*\top} (\xb_k + \deb) \big) \notag\\
    &= L_k^{p,r}(\wb_k) - L_k^{p,r}(\wb^*).\label{eq:adv.training.ip.lb.sgd}
\end{align}
For the gradient upper bound, under the assumption that $\ell$ is $L$-Lipschitz,
\begin{align*}
    \big\|  \ell'\big(y_k \wb^\top(\xb_k + \deb_k) \big) y_k (\xb_k + \deb_k) \big\|_2^2 &\leq  \| \ell'\big(y_k \wb^\top(\xb_k + \deb_k) \big) y_k (\xb_k + \deb_k) \|_2^2 \\
    &\leq  L^2 \norm{\xb_k + \deb_k}_2^2 \\
    &\leq 2 L^2  (\norm{\xb_k}_2^2 + \norm{\deb_k}_2^2)\\
    &\leq 2 L^2 \sup_{\xb \sim \calD_x } \norm{\xb}_2^2 + 2 L^2 \sup_{\norm{\deb}_p \leq r} \norm{\deb}_2^2 \\
    &\leq 2 L^2 \sup_{\xb \sim \calD_x} \norm{\xb}_p^2 \cdot \frac{\norm{\xb}_2^2}{\norm{\xb}_p^2} + 2 L^2 \sup_{\norm{\deb}_p \leq r} \norm{\deb}_p^2 \cdot \frac{\norm{\deb}_2^2}{\norm{\deb}_p^2}\\
    &\leq \begin{cases} 2 L^2(1 + r), & p \leq 2,\\
    2L^2(d + rd),& p > 2.\end{cases}
\end{align*}
In the first inequality, we use Jensen's inequality.  In the second we use that $\ell$ is $L$-Lipschitz.  The third inequality follows by Young's inequality.  In the last, we use that $\norm{x}_p\leq 1$ and that $p \mapsto \norm{\xb}_p$ is a decreasing function for fixed $\xb$, together with the bound $\norm{\xb}_2/\norm{\xb}_\infty \leq \sqrt{d}$.   Assuming without loss of generality that $r \leq 1$, this shows that 
\begin{equation}\label{eq:adv.training.norm.bound.sgd}
    \big\| \ell'\big(y_k \wb^\top(\xb_k + \deb_k) \big) y_k (\xb_k + \deb_k) \big\|_2^2 \leq H := \begin{cases} 4 L^2, & p \leq 2,\\
    4L^2d,& p > 2.\end{cases}
\end{equation}
Putting \eqref{eq:adv.training.ip.lb.sgd} and \eqref{eq:adv.training.norm.bound.sgd} together, we have for $\eta \leq \eps H/4$,
\begin{align}\label{eq:onestep_norm_decrease.sgd}
    \|\wb_k - \wb^*\|_2^2 - \|\wb_{k+1} - \wb^*\|_2^2 \geq 2 \eta(L_k^{p,r}(\wb_k) - L_k^{p,r}(\wb^*) - \eta^2 H \geq 2 \eta(L_k^{p,r}(\wb_k) - L_k^{p,r}(\wb^*) - \eps/2).
\end{align}
Telescoping the above sum over $k$, we get
\begin{equation*}
    \frac 1 K \sum_{k=1}^{K} L_k^{p,r}(\wb_k)\leq \frac{1}{K}\sum_{k=1}^{K}L_k^{p,r}(\wb^*)+ \frac{\norm{\wb_1 - \wb^*}_2^2}{\eta K} + \eps/2 = L_{\cS}^{p,r}(\wb^*)+ \frac{\norm{\wb_1 - \wb^*}_2^2}{\eta K} + \eps/2.
\end{equation*}
Taking $K = 2\eps^{-1} \eta^{-1} \norm{\wb_1 - \wb^*}_2^2$, we are able to show that 
\begin{align}\label{eq:regret_bound.sgd}
\frac{1}{K}\sum_{k=1}^{K}L_k^{p,r}(\wb_k)\le L_{\cS}^{p,r}(\wb^*) +\epsilon.
\end{align}
This completes the proof.

\end{proof}

Later we will give the following lemma which shows that the empirical robust risk $ L_{\cS}^{p,r}(\wb^*)$ can be upper bounded by $ O(L_{\cD}^{p,r}(\wb^*))+\tilde O(1/K)$. 

\begin{lemma}\label{lemma:concentration_w*}
Suppose $\|\wb^*\|_q\le \rho$ and $\|\xb\|_p\le 1$, then for any $\delta\in(0,1)$ we have with probability at least $1-\delta$,
\begin{align*}
L_\cS^{p,r}(\wb^*) \le 2L_\cD^{p,r}(\wb^*) + \frac{\big(\ell(0)+L\rho(1+r)\big)\log(1/\delta)}{K}
\end{align*}
\end{lemma}
\begin{proof}
We will use Lemma A.5 in \citet{frei2020halfspace} to prove Lemma \ref{lemma:concentration_w*}. In particular, we only need to calculate the upper bound $\ell(y\wb^{*\top}(\xb+\deb))$ is upper for any $\xb$ and $y$ satisfying $\|\xb\|_p\le 1$. In fact, we can get
\begin{align*}
\ell(y\wb^{*\top}(\xb+\deb))\le \ell(0) + L|\wb^{*\top}(\xb+\deb)| \le \ell(0) + L(\|\wb^*\|_q\|\xb\|_p + r\|\wb^*\|_q)\le \ell(0)+L\rho(1+r).
\end{align*}
where the first inequality is due to $\ell(\cdot)$ is $L$-Lipschitz, the second inequality is by Holder's inequality, and the last inequality is due to $\|\xb\|_p\le 1$ and $\|\wb^*\|_q\le \rho$.
Then by Lemma A.5 in \citet{frei2020halfspace}, we can get  that with probability at least $1-\delta$,
\begin{align*}
L_\cS^{p,r}(\wb^*) \le 2L_\cD^{p,r}(\wb^*) + \frac{\big[\ell(0)+L\rho(1+r)\big]\log(1/\delta)}{K}.
\end{align*}
This completes the proof.

\end{proof}

Then we will use similar approach in \citet{frei2020halfspace} to get a high-probability guarantees for the output of SGD. In particular, we will use the following lemma to show that $\big\{\big[\ell'(y_k\wb_k^\top(\xb_k+\deb_k))\big]^2\big\}$ concentrates at rate $O(1/K)$ for any fixed stochastic gradient descent iterates $\{\wb_k\}$.
\begin{lemma}[Lemma A.4 in \citet{frei2020halfspace}]\label{eq:concentration_grad}
Under the same assumption in Lemma \ref{lemma:convergence_advtraining_SGD}, then for any $\delta\in(0,1)$, with probability at least $1-\delta$,
\begin{align*}
\frac{1}{K}\sum_{k=1}^{K} \EE_{(\xb,y)\sim\cD} \Big[\big[\ell'(y\wb_k^\top\xb - r\|\wb_k\|_q)\big]^2\Big]\le \frac{4}{K} \sum_{k=1}^K \big[\ell'(y_k\wb_k^\top(\xb_k+\deb_k))\big]^2 + \frac{4L^2 \log(1/\delta)}{K},
\end{align*}
\end{lemma}
Here we slightly modify the original version of Lemma A.4 in \citet{frei2020halfspace} by introducing the adversarial examples. In particular, we use the fact that $\deb_k$ is the optimal perturbation with respect to the data $(\xb,y)$ and the model $\wb_k$ (see \eqref{eq:delta.star.def}).  Therefore, we have $y_k\wb_k^\top\deb_k=-r\|\wb_k\|_q$ and thus
\begin{align*}
\EE_{(\xb_k,y_k)\sim\cD}\Big[\big[\ell'(y_k\wb_k^\top(\xb_k+\deb_k))\big]^2\Big] & = \EE_{(\xb_k,y_k)\sim\cD}\Big[\big[\ell'(y_k\wb_k^\top\xb_k - r\|\wb_k\|_q)\big]^2\Big] \notag\\
&=  \EE_{(\xb,y)\sim\cD} \Big[\big[\ell'(y\wb_k^\top\xb - r\|\wb_k\|_q)\big]^2\Big].
\end{align*}
Then Lemma A.4 in \citet{frei2020halfspace} is applicable since it only requires the Lipschitzness of $\ell(\cdot)$ and the fact that $\Big\{\big[\ell'(y_k\wb_k^\top(\xb_k+\deb_k))\big]^2-\EE_{(\xb,y)\sim\cD} \Big[\big[\ell'(y\wb_k^\top\xb - r\|\wb_k\|_q)\big]^2\Big]\Big\}$ is a martingale difference sequence.

With this, we can show that for any smooth loss function $\ell(\cdot)$, we are able to get a high-probability bound on the robust error $\min_{k\le K}\err_{\cD}^{p,r}(\wb_k)$. 

Now we are ready to complete the proof of Theorem \ref{thm:guarantee_SGD}.
\begin{proof}[Proof of Theorem \ref{thm:guarantee_SGD}]
Since $\ell(\cdot)$ is $M$-smooth, we have $[\ell'(z)]^2\le 4M\ell(z)$ for all $z\in\RR$. Then applying Lemma \ref{eq:concentration_grad}, we have with probability at least $1-\delta/2$,
\begin{align*}
\frac{1}{K}\sum_{k=1}^{K} \EE_{(\xb,y)\sim\cD} \Big[\big[\ell'(y\wb_k^\top\xb - r\|\wb_k\|_q)\big]^2\Big]&\le \frac{4}{K} \sum_{k=1}^K \big[\ell'(y_k\wb_k^\top(\xb_k+\deb_k))\big]^2 + \frac{4L^2 \log(2/\delta)}{K}\notag\\
&\le \frac{16M}{K} \sum_{k=1}^K L_k^{p,r}(\wb_k)  + \frac{4L^2 \log(2/\delta)}{K}.
\end{align*}
Further applying Lemma \ref{lemma:convergence_advtraining_SGD} gives
\begin{align}\label{eq:sgdbound_pathaverage_gradsquare}
\frac{1}{K}\sum_{k=1}^{K} \EE_{(\xb,y)\sim\cD} \Big[\big[\ell'(y\wb_k^\top\xb - r\|\wb_k\|_q)\big]^2\Big]\le 16M L_\cS^{p,r}(\wb^*)  + 16M\epsilon + \frac{4L^2 \log(2/\delta)}{K}
\end{align}
where $\wb^*\in\RR^d$ and $\epsilon\in(0,1)$ are arbitrary. 

Since $\ell(\cdot)$ is convex and decreasing, it is easy to verify that $[\ell'(z)]^2$ is decreasing. Applying Markov's inequality gives
\begin{align*}
\err_{\cD}^{p,r}(\wb) &= \PP[y\wb^\top\xb- r\|\wb\|_q\le 0] = \PP\Big[\big[\ell'(y\wb^\top\xb- r\|\wb\|_q)\big]^2\ge [\ell'(0)]^2\Big]\\
&\le [\ell'(0)]^{-2}\cdot\EE\Big[\big[\ell'(y\wb^\top\xb- r\|\wb\|_q)\big]^2\Big].
\end{align*}
Therefore, substituting the above inequality into \ref{eq:sgdbound_pathaverage_gradsquare} gives
\begin{align}\label{eq:bound_average_error_sgd}
\frac{1}{K}\sum_{k=1}^K \err_{\cD}^{p,r}(\wb_k)&\le \frac{[\ell'(0)]^{-2}}{K}\sum_{k=1}^K \EE\Big[\big[\ell'(y\wb^\top\xb- r\|\wb\|_q)\big]^2\Big]\notag\\
&\le [\ell'(0)]^{-2}\cdot\bigg(16M L_\cS^{p,r}(\wb^*)  + 16M\epsilon + \frac{4L^2 \log(2/\delta)}{K}\bigg).
\end{align}
Moreover, by Lemma \ref{lemma:concentration_w*} we have with probability at least $1-\delta/2$,
\begin{align*}
L_\cS^{p,r}(\wb^*) \le 2L_\cD^{p,r}(\wb^*) + \frac{\big(\ell(0)+L\rho(1+r)\big)\log(2/\delta)}{K}.
\end{align*}
Substituting the above inequality into \eqref{eq:bound_average_error_sgd} and using the fact that $\min_{k\le K}\err_{\cD}^{p,r}(\wb_k)\le K^{-1}\sum_{k=1}^K \err_{\cD}^{p,r}(\wb_k)$, we have with probability $1-\delta$ that
\begin{align*}
\min_{k\le K}\err_{\cD}^{p,r}(\wb_k)\le[\ell'(0)]^{-2}\cdot\bigg(32M L_\cD^{p,r}(\wb^*) + \frac{16M\big(\ell(0)+L\rho(1+r)\big)\log(1/\delta)}{K} + 16M\epsilon + \frac{4L^2 \log(1/\delta)}{K}\bigg)
\end{align*}
which completes the proof.


\end{proof}

\section{Proofs for Nonconvex Sigmoidal Loss}
We first restate some general calculations which will be frequently used in the subsequent analysis.
Let $h(\wb,\xb) =  \wb^\top\xb/\|\wb\|_q$ be the prediction of the model and $S = \{(\xb,y): y= \text{sgn}(\wb^{*\top}\xb)\}$ be the set of data which can be correctly classified by $\wb^*$ without perturbation, we have
\begin{align*}
\nabla_{\wb}  L_{\cD}^{p,r}(\wb) &= \EE_{(\xb,y)\sim \cD}\big[\ell'(yh(\wb,\xb+\deb))\cdot y\cdot \nabla_{\wb}h(\wb,\xb+\deb)\big]\notag\\
& = \EE_{(\xb,y)\sim \cD}\big[\ell'(yh(\wb,\xb+\deb))\cdot y\cdot \nabla_{\wb}h(\wb,\xb+\deb)\cdot (\ind(S) + \ind(S^c))\big].
\end{align*}
Moreover, note that $\deb$ is the optimal $\ell_p$ adversarial perturbation corresponding to the model parameter $\wb$, it can be calculated that $yh(\wb,\xb+\deb) = y\wb^\top\xb/\|\wb\|_q - r$. Then it follows that
\begin{align*}
yh(\wb,\xb+\deb) = \left\{
\begin{array}{cc}
  \text{sgn}(\wb^{*\top}\xb)\cdot\frac{\wb^\top\xb}{\|\wb\|_q}-r   & \xb\in S \\
   -\text{sgn}(\wb^{*\top}\xb)\cdot\frac{\wb^\top\xb}{\|\wb\|_q}-r  & \xb\in S^c
\end{array}
\right.
\end{align*}
Let $g_S(\wb^*,\wb;\xb) = \ell'(\text{sgn}(\wb^{*\top}\xb)\cdot \wb^\top\xb/\|\wb\|_q-r )$ and $g_{S^c}(\wb^*,\wb;\xb) = \ell'(-\text{sgn}(\wb^{*\top}\xb)\cdot \wb^\top\xb/\|\wb\|_q-r )$, the gradient $\nabla_{\wb}L_{\cD}^{p,r}(\wb)$ can be rewritten as
\begin{align*}
\nabla_{\wb}L_{\cD}^{p,r}(\wb)&=\EE_{(\xb,y)\sim \cD}\big[\ell'(yh(\wb,\xb+\deb))\cdot \text{sgn}(\wb^{*\top}\xb)\cdot \nabla_{\wb}h(\wb,\xb+\deb)\cdot  \ind_{S}(\xb)\big]\notag\\
& \qquad - \EE_{(\xb,y)\sim \cD}\big[\ell'(yh(\wb,\xb+\deb))\cdot \text{sgn}(\wb^{*\top}\xb)\cdot \nabla_{\wb}h(\wb,\xb+\deb)\cdot  \ind(S^c)\big]\notag\\
& = \EE_{(\xb,y)\sim \cD}\big[g_{S}(\wb^*,\wb;\xb)\cdot \nabla_{\wb}h(\wb,\xb+\deb)\cdot  \ind_{S}(\xb)\big]\notag\\
& \qquad - \EE_{(\xb,y)\sim \cD}\big[g_{S^c}(\wb^*,\wb;\xb)\cdot \text{sgn}(\wb^{*\top}\xb)\cdot \nabla_{\wb}h(\wb,\xb+\deb)\cdot  \ind(S^c)\big]\notag\\
& = \EE_{(\xb,y)\sim \cD}\big[g_{S}(\wb^*,\wb;\xb)\cdot \nabla_{\wb}h(\wb,\xb+\deb)\big]\notag\\
& \qquad - \EE_{(\xb,y)\sim \cD}\big[\big(g_{S}(\wb^*,\wb;\xb)+g_{S^c}(\wb^*,\wb;\xb)\big)\cdot \text{sgn}(\wb^{*\top}\xb)\cdot \nabla_{\wb}h(\wb,\xb+\deb)\cdot  \ind(S^c)\big]
\end{align*}
Then it follows that
\begin{align}\label{eq:formula_inner_product}
\wb^{*\top} \nabla L_{\cD}^{p,r}(\wb) &= \underbrace{\EE_{(\xb,y)\sim \cD}\big[g_{S}(\wb^*,\wb;\xb)\cdot \text{sgn}(\wb^{*\top}\xb)\cdot \wb^{*\top}\nabla_{\wb}h(\wb,\xb+\deb) \big]}_{I_1}\notag\\
& - \underbrace{\EE_{(\xb,y)\sim \cD}\big[\big(g_{S^c}(\wb^*,\wb;\xb)+g_{S^c}(\wb^*,\wb;\xb)\big)\cdot \text{sgn}(\wb^{*\top}\xb)\cdot \wb^{*\top}\nabla_{\wb}h(\wb,\xb+\deb)\cdot \ind(S^c)\big]}_{I_2}
\end{align}

Note that
\begin{align}\label{eq:grad.h.formula}
\nabla_{\wb}h(\wb,\xb+\deb) = \bigg(\Ib - \frac{\bar \wb \wb^\top}{\|\wb\|_q^{q}}\bigg)\frac{\xb+\deb}{\|\wb\|_q},
\end{align}
where $\bar w_j = |w_j|^{q-1}\text{sgn}(w_j)$. Therefore, it holds that 
\begin{align*}
\wb^{*\top}\nabla_{\wb}h(\wb,\xb+\deb) = \frac{\wb^{*\top}(\xb+\deb)}{\|\wb\|_q} - \frac{\wb^\top(\xb+\deb)\cdot \bar \wb^{\top}\wb^*}{\|\wb\|_q^{q+1}}.
\end{align*}
Consider the optimal adversarial perturbation for the classifier $\wb$, we have $\deb =  - ry\bar \wb/\|\wb\|_q^{q-1}$.  Thus it follows that
\begin{align}
\wb^{*\top}\nabla_{\wb}h(\wb,\xb+\deb) &=\frac{\wb^{*\top}\xb}{\|\wb\|_q}- \frac{r y\wb^{*\top}\bar \wb}{\|\wb\|_q^q} - \frac{(\wb^\top\xb/\|\wb\|_q - r y)\cdot \bar \wb^\top \wb^*}{\|\wb\|_q^q}\notag\\
&= \frac{\wb^{*\top}\xb}{\|\wb\|_q} - \frac{\wb^\top\xb\cdot \bar \wb^\top \wb^*}{\|\wb\|_q^{q+1}}:= \tilde\wb^\top\xb,\label{eq:strongest.attacker}
\end{align}
where $\tilde\wb = \wb^*/\|\wb\|_q - (\bar\wb^\top\wb^*)\wb/\|\wb\|_q^{q+1}$.

\subsection{Proof of Lemma \ref{lemma:upperbound_opt_grad_product_main}}\label{sec:upperbound_inner_product}


\begin{proof}[Proof of Lemma \ref{lemma:upperbound_opt_grad_product_main}]
We will focus on the $2$-dimensional space spanned by the vectors $\wb$ and $\wb^*$ (or $\tilde \wb$ since these three vectors lie in the same $2$-dimensional space). Without loss of generality, we assume $\wb=\|\wb\|_2 \eb_2$.  We further define the set $G:=\{\xb:\text{sgn}(\wb^{*\top}\xb) = \text{sgn}(\tilde \wb^\top\xb)\}$, which is marked as the shaded region in Figure \ref{fig:diagram}. According to \eqref{eq:formula_inner_product}, the entire proof will be decomposed into three parts: upper bounding $I_1$, upper bounding $|I_2|$, and combining these two bounds to get the desired results. In the remaining proof we will the short-hand notations $\theta$ and $\theta'$ to denote $\theta(\wb)$ and $\theta'(\wb)$ respectively.

Without loss of generality, we consider the case that $\angle(\wb,\wb^*)\in(0, \pi/2)$ and the case of $\angle(\wb,\wb^*)\in(0, \pi/2)$ follows similarly by conducting the transformation $\wb\leftarrow -\wb$.

\paragraph{Upper bounding $I_1$.}
Note that within the set $G$, we have $\text{sgn}(\wb^{*\top}\xb)=\text{sgn}(\tilde \wb^\top\xb)$ thus $I_1$ can be decomposed as follows accordingly,
\begin{align}\label{eq:formula_I1_lp}
I_1 = \underbrace{\EE_{(\xb,y)\sim \cD}\big[g_{S}(\wb^*,\wb;\xb)\cdot | \tilde \wb^\top\xb |\cdot I_{G}(\xb))\big]}_{I_3}-\underbrace{\EE_{(\xb,y)\sim \cD}\big[g_{S}(\wb^*,\wb;\xb)\cdot | \tilde \wb^\top\xb |\cdot I_{G^c}(\xb))\big]}_{I_4}.
\end{align}
where 
\begin{align}\label{eq:formula_g_sc}
g_{S}(\wb^*,\wb;\xb) &= \ell'(\text{sgn}(\wb^{*\top}\xb)\cdot \wb^\top\xb-r) = -\frac{e^{-|\text{sgn}(\wb^{*\top}\xb)\cdot \wb^\top\xb-r|/\sigma}}{\sigma} = -\frac{e^{-|\text{sgn}(\wb^{*\top}\xb)\cdot l\|\wb\|_2\sin\phi-r|/\sigma}}{\sigma},
\end{align}
where the first equality is due to the assumption that $\|\wb\|_q=1$ and the last inequality is due to the fact that $\wb=\|\wb\|_2\eb_2$ so that $\wb^\top\xb = l\|\wb\|_2\sin\phi$.

Let $\bar\xb = (l\cos\phi,l\sin\phi)$ be the projection of $\xb$ onto the space spanned by $\wb$ and $\wb^*$. Then under Assumption \ref{assump:isotropic} we know that the distribution $\cD_x$ is isotropic, which implies that the probability density function of $\bar \xb$, denoted by $p(\bar\xb)$, can be written as $p(\bar\xb)=p(l)/(2\pi)$, where $p(l)$ is the probability density function with respect to the length of $\bar\xb$. Then based on the formula of $g_{S}(\wb^*, \wb;\xb)$ derived in \eqref{eq:formula_g_sc} and the fact that $\tilde \wb^\top\xb = \|\tilde \wb\|_2\sin(\theta'-\phi)$, we have the following regarding $I_3$,
\begin{align}\label{eq:formula_I3}
I_3 &= -\frac{1}{\sigma}\int_{0}^\infty \bigg(\int_{-\theta}^{\theta'} p(\bar \xb) l^2\|\tilde \wb\|_2|\sin(\theta'-\phi)| e^{-|l\|\wb\|_2\sin\phi-r|/\sigma}\dd l\dd \phi \notag\\
&\qquad + \int_{\pi-\theta}^{\pi+\theta'} p(\bar \xb) l^2\|\tilde \wb\|_2|\sin(\theta'-\phi)| e^{-|l\|\wb\|_2\sin\phi+r|/\sigma}\dd l\dd \phi\bigg)\notag\\
& = -\frac{\|\tilde \wb\|_2}{\pi\sigma}\int_{0}^\infty p(l)l^2 \dd l \underbrace{\int_{-\theta}^{\theta'} |\sin(\theta'-\phi)| e^{-|l\|\wb\|_2\sin\phi-r|/\sigma}\dd \phi}_{I_5},
\end{align}
where the equality holds since $\sin(\pi+\phi) = -\sin\phi$ and $|\sin(\theta-\phi)| = |\sin(\theta-\phi-\pi)|$ for all $\phi$.
Note that when $\phi\le 0$, we have
\begin{align*}
\frac{\sin(\theta'-\phi)}{\cos\phi} = \frac{\sin\theta'\cos\phi - \sin\phi\cos\theta'}{\cos\phi}\ge \sin\theta',
\end{align*}
where the last inequality holds since $\sin\phi\cos\theta'\le 0$ for all $\phi\in(-\pi/2,0)$ and $\theta'\in(0,\pi/2)$. Therefore, we have the following lower bound on the term $I_5$,
\begin{align*}
I_5&\ge \int_{-\theta}^{0}\sin(\theta'-\phi) e^{-|l\|\wb\|_2\sin\phi-r|/\sigma}\dd \phi\notag\\
& \ge \sin\theta'\int_{-\theta}^{0} \cos\phi e^{(l\|\wb\|_2\sin\phi-r)/\sigma}\dd \phi\notag\\
&=\frac{\sin\theta'\cdot\sigma e^{-r/\sigma}}{l\|\wb\|_2}\cdot \big(1-e^{-l\|\wb\|_2\sin\theta/\sigma}\big),
\end{align*}
where in the second inequality we  use the fact that $l\|\wb\|_2\sin\phi\le 0$ for all $\phi\in(-\theta,0)$.
Plugging the above bound of $I_3$ into \eqref{eq:formula_I3} gives the following upper bound on $I_3$,
\begin{align}\label{eq:upperbound_I3}
I_3&\le -\frac{\|\tilde \wb\|_2\cdot\sin\theta'\cdot e^{-r/\sigma}}{\pi\|\wb\|_2}\cdot \int_{0}^\infty p(l)l (1-e^{-l\|\wb\|_2\sin\theta/\sigma}) \dd l\notag\\
&\le -\frac{2U'\|\tilde \wb\|_2\cdot\sin\theta'\cdot e^{-r/\sigma}}{\|\wb\|_2 UR}\cdot \int_{0}^R l^2 (1-e^{-R\|\wb\|_2\sin\theta/\sigma}) \dd l\notag\\
& = -\frac{2U'R^2\|\tilde \wb\|_2\cdot\sin\theta'\cdot e^{-r/\sigma}}{3\|\wb\|_2 }\cdot  (1-e^{-R\|\wb\|_2\sin\theta/\sigma}),
 \end{align}
where for the second inequality 
we use the fact that the function $f(x)=(1-e^{-ax})/x$ is strictly decreasing with respect to $x$ for any $a\ge 0$ so that we have $1- e^{-l\sin\theta/\sigma}\ge l(1- e^{-R\sin\theta/\sigma})/R$ for any $l\in(0, R]$, besides we also use the fact that $\cD_x$ is $(U',R)$-anti-anti-concentration so that $p(\bar\xb)=p(l)/(2\pi)\ge U'$. 
Then we will focus on lower bounding $I_4$. Let $\theta_{\min} = \min\{\theta, \theta', \pi/3\}$, we have
\begin{align}\label{eq:lowerbound_I4}
|I_4| &= \frac{\|\tilde \wb\|_2}{\pi\sigma}\int_{0}^\infty p(l)l^2 \dd l\int_{\theta'}^{\pi-\theta} |\sin(\phi-\theta')| e^{-|l\|\wb\|_2\sin\phi-r|/\sigma}\dd \phi\notag\\
&\le \frac{\|\tilde \wb\|_2}{\pi\sigma}\int_{0}^\infty p(l)l^2 \dd l\int_{\theta_{\min}}^{\pi-\theta_{\min}}  e^{-|l\|\wb\|_2\sin\phi-r|/\sigma}\dd \phi\notag\\
&= \frac{2\|\tilde \wb\|_2}{\pi\sigma}\int_{0}^\infty p(l)l^2 \dd l\underbrace{\int_{\theta_{\min}}^{\pi/2}  e^{-|l\|\wb\|_2\sin\phi-r|/\sigma}\dd \phi}_{I_6},
\end{align}
where and the inequality follows from the fact that $|\sin(\phi-\theta')|\le 1$ and $\sin(\phi) = \sin(\pi-\phi)$.
Note that $I_6$ can be further upper bounded as follows,
\begin{align*}
I_6& = \int_{\theta_{\min}}^{\pi/3}  e^{-|l\|\wb\|_2\sin\phi-r|/\sigma}\dd \phi + \int_{\pi/3}^{\pi/2}  e^{-|l\|\wb\|_2\sin\phi-r|/\sigma}\dd \phi\notag\\
&\le \underbrace{2\int_{\theta_{\min}}^{\pi/3} \cos\phi e^{-|l\|\wb\|_2\sin\phi-r|/\sigma}\dd \phi}_{I_7} + \underbrace{\int_{\pi/3}^{\pi/2}  e^{-|l\|\wb\|_2\sin\phi-r|/\sigma}\dd \phi}_{I_8},
\end{align*}
where the inequality holds due to the fact that $2\cos\phi\ge 1$ when $\phi\in[\theta_{\min},\pi/3]$. 
Regarding $I_7$, we will consider three cases: (1) $l\le r/(\sin(\pi/3)\|\wb\|_2)$, (2) $r/(\sin(\pi/3)\|\wb\|_2)\le l\le r/(\sin(\theta_{\min})\|\wb\|_2)$ and (3) $l\ge r/(\sin(\theta_{\min})\|\wb\|_2)$.
\begin{enumerate}
    \item Regarding the first case $l\le r/(\sin(\pi/3)\|\wb\|_2)$, it holds that
\begin{align*}
I_7\le 2\int_{\sin(\theta_{\min})}^{\sin(\pi/3)}  e^{(l\|\wb\|_2z-r)/\sigma}\dd \phi = \frac{2\sigma e^{-r/\sigma}}{l\|\wb\|_2}\cdot\big(e^{l\|\wb\|_2\sin(\pi/3)/\sigma}-e^{l\|\wb\|_2\sin(\theta_{\min})/\sigma}\big)\le \frac{2\sigma }{l\|\wb\|_2}.
\end{align*}
\item Regarding the second case $r/(\sin(\pi/3)\|\wb\|_2)\le l\le r/(\sin(\theta_{\min})\|\wb\|_2)$, we have
\begin{align*}
I_7&\le 2\int_{\sin(\theta_{\min})}^{r/(l\|\wb\|_2)}  e^{(l\|\wb\|_2z-r)/\sigma}\dd \phi + 2\int_{r/(l\|\wb\|_2)}^{1} e^{(r-l\|\wb\|_2z)/\sigma}\dd \phi \notag\\
& = \frac{2\sigma }{l\|\wb\|_2}\cdot\big(2-e^{(l\|\wb\|_2\sin(\btheta_{\min})-r)/\sigma}-e^{(r-l\|\wb\|_2)\sigma}\big)\notag\\
&\le \frac{4\sigma }{l\|\wb\|_2}.
\end{align*}
\item Regarding the third case $l\ge r/(\sin(\theta_{\min})\|\wb\|_2)$, we have
\begin{align*}
I_7&\le  2\int_{\sin(\theta_{\min})}^{1} e^{(r-l\|\wb\|_2z)/\sigma}\dd \phi \le  \frac{2\sigma e^{r/\sigma} }{l\|\wb\|_2}\cdot e^{-l\sin(\theta_{\min})\|\wb\|_2/\sigma}
\end{align*}
\end{enumerate}

\noindent For $I_8$, it is easy to see that
\begin{align*}
e^{-|l\|\wb\|_2\sin\phi-r|/\sigma}&\le \left\{
\begin{array}{ll}
   1  &  l\le r/(\sin(\pi/3)\|\wb\|_2)\\
    e^{-(l\|\wb\|_2\sin\phi-r)/\sigma}     &  l> r/(\sin(\pi/3)\|\wb\|_2),
\end{array}
\right.\notag\\
&\le \left\{
\begin{array}{ll}
   1  &  l\le r/(\sin(\pi/3)\|\wb\|_2)\\
    e^{-(l\|\wb\|_2\sin(\phi/2)-r)/\sigma}     &  l> r/(\sin(\pi/3)\|\wb\|_2),
\end{array}
\right.
\end{align*}
which implies that if $l\le r/(\sin(\pi/3)\|\wb\|_2)$,
\begin{align*}
I_8= \int_{\pi/3}^{\pi/2}  e^{-|l\|\wb\|_2\sin\phi-r|/\sigma}\dd \phi\le \pi/6
\end{align*}
and if $l\ge r/(\sin(\pi/3)\|\wb\|_2)$
\begin{align*}
 I_8&\le  \int_{\pi/3}^{\pi/2}  e^{-(l\|\wb\|_2\sin(\phi/2)-r)/\sigma}   \dd \phi \notag\\
 &\le 2\int_{\pi/3}^{\pi/2} \cos(\phi/2) e^{-(l\|\wb\|_2\sin(\phi/2)-r)/\sigma}   \dd \phi \notag\\
 &\le 4\int_{\sin(\pi/6)}^{\sin(\pi/4)} e^{-(l\|\wb\|_2z-r)/\sigma}\dd z\notag\\
&\le \frac{4\sigma e^{r/\sigma}}{l\|\wb\|_2 }\cdot e^{-l\|\wb\|_2\sin(\pi/6)/\sigma}   
\end{align*}

Combining the above results, define $\bar\theta_{\min}=\min\{\theta_{\min},\pi/6\}=\min\{\theta, \theta', \pi/6\}$, we have
the following bounds on $I_6$,
\begin{align*}
I_6\le \left\{
\begin{array}{ll}
  \frac{2\sigma}{l\|\wb\|_2}+\frac{\pi}{6}   &  l\le r/(\sin(\pi/3)\|\wb\|_2)\\
   \frac{8\sigma e^{r/\sigma}}{l\|\wb\|_2}  &  r/(\sin(\pi/3)\|\wb\|_2)\le l\le r/(\sin(\theta_{\min})\|\wb\|_2)\\
   \frac{6\sigma e^{r/\sigma} }{l\|\wb\|_2}\cdot e^{-l\sin(\bar\theta_{\min})\|\wb\|_2/\sigma}   & l\ge r/(\sin(\theta_{\min})\|\wb\|_2)
\end{array}
\right.
\end{align*}
Plugging this bound into \eqref{eq:lowerbound_I4}, we have
\begin{align*}
|I_4|
&\le \|\tilde \wb\|_2\cdot\bigg(\int_{0}^{r/(\sin(\pi/3)\|\wb\|_2)}\bigg(\frac{4l}{\pi\|\wb\|_2}+\frac{l^2}{3\sigma}\bigg)p(l)\dd l + \int_{r/(\sin(\pi/3)\|\wb\|_2)}^{r/(\sin(\theta_{\min})\|\wb\|_2)}\frac{16l e^{r/\sigma}}{\pi\|\wb\|_2}p(l)\dd l\notag\\
&\qquad + \int_{r/(\sin(\theta_{\min})\|\wb\|_2)}^{\infty}\frac{12l e^{r/\sigma}}{\pi\|\wb\|_2}e^{-l\sin(\bar\theta_{\min})\|\wb\|_2/\sigma}p(l)\dd l\bigg)\notag\\
&\le \|\tilde \wb\|_2\cdot\bigg(\int_{0}^{r/(\sin(\pi/3)\|\wb\|_2)}\bigg(\frac{4l}{\pi\|\wb\|_2}+\frac{l^2}{3\sigma}\bigg)p(l)\dd l + \int_{0}^{r/(\sin(\theta_{\min})\|\wb\|_2)}\frac{16l e^{r/\sigma}}{\pi\|\wb\|_2}p(l)\dd l\notag\\
&\qquad + \int_{0}^{\infty}\frac{12l e^{r/\sigma}}{\pi\|\wb\|_2}e^{-l\sin(\bar\theta_{\min})\|\wb\|_2/\sigma}p(l)\dd l\bigg).
\end{align*}
Note that $\int_{0}^\infty xe^{-a x}\dd x = 1/a^2$ for any $a\ge 0$. Additionally, by Assumption \ref{assump:isotropic} we have $p(l)\le 2\pi U$ since $\cD_x$ is $U$-anti-concentration. Then we can get that
\begin{align*}
|I_4|&\le U\|\tilde \wb\|_2\cdot\bigg(\frac{8r^2}{\sin^2(\pi/3)\|\wb\|_2^3}+\frac{2\pi r^3}{3\sin^3(\pi/3)\sigma\|\wb\|_2^3} +\frac{32r^2e^{r/\sigma}}{\sin^2(\theta_{\min})\|\wb\|_2^3}+\frac{24\sigma^2e^{r/\sigma}}{\sin^2(\bar\theta_{\min})\|\wb\|_2^3}\bigg)\notag\\
&\le U\|\tilde \wb\|_2e^{r/\sigma}\cdot\frac{4r^3/\sigma+ 40r^2+24\sigma^2}{\sin^2(\bar\theta_{\min})\|\wb\|_2^3}.
\end{align*}
where in the last inequality we use the fact that $\theta_{\min}\le \bar\theta_{\min}$.
Plugging the above inequality and \eqref{eq:upperbound_I3} into \eqref{eq:formula_I1_lp}, we obtain
\begin{align*}
I_1 = I_3-I_4\le -\frac{2U'R^2\|\tilde \wb\|_2\cdot\sin\theta'\cdot e^{-r/\sigma}}{3\|\wb\|_2}\cdot  (1-e^{-R\|\wb\|_2\sin\theta/\sigma}) + U\|\tilde \wb\|_2e^{r/\sigma}\cdot\frac{4r^3/\sigma+ 40r^2+24\sigma^2}{\sin^2(\bar\theta_{\min})\|\wb\|_2^3}.
\end{align*}
Then set $\sigma = r$, we have
\begin{align*}
I_1&\le  -\frac{2U'R^2\|\tilde \wb\|_2\cdot\sin\theta'\cdot e^{-1}}{3\|\wb\|_2 }\cdot  (1-e^{-R\|\wb\|_2\sin\theta/r}) + \frac{68er^2U\|\tilde \wb\|_2}{\sin^2(\bar\theta_{\min})\|\wb\|_2^3}.
\end{align*}
Then we have if
\begin{align}\label{eq:condition_theta_min}
\sin(\bar\theta_{\min})\ge \max\bigg\{\frac{4r}{R\|\wb\|_2},\frac{100 r \sqrt{U/U'} }{R\|\wb\|_2\sin^{1/2}(\theta') }\bigg\},
\end{align}
it holds that 
\begin{align*}
I_1\le -\frac{3U'R^2\|\tilde \wb\|_2\cdot\sin\theta'\cdot e^{-1}}{5\|\wb\|_2}.
\end{align*}
Moreover, note that $\bar\theta_{\min} = \min\{\theta,\theta',\pi/6\}$, thus if the perturbation level satisfies
\begin{align*}
r\le \min\bigg\{\frac{R\|\wb\|_2}{8}, \frac{R\|\wb\|_2\sin^{1/2}(\theta')}{200 U},\frac{R\|\wb\|_2\sin(\theta')}{4}, \frac{R\|\wb\|_2\sin^{3/2}(\theta')}{100 U}\bigg\} = O\big(R\|\wb\|_2\sin^{3/2}(\theta')\big),
\end{align*}
the condition \eqref{eq:condition_theta_min} is equivalent to 
\begin{align*}
\sin(\theta)\ge \max\bigg\{\frac{4r}{R\|\wb\|_2},\frac{100  r\sqrt{U/U'}}{R\|\wb\|_2\sin^{1/2}(\theta') }\bigg\},
\end{align*}
\paragraph{Upper bounding $|I_2|$.}
In the sequel we will focus on bounding the term $I_2$. By Cauchy-Sharwtz inequality, we have
\begin{align*}
|I_2| \le \sqrt{\underbrace{\EE_{\xb}\big[\big(g_{S}(\wb^*,\wb;\xb)+g_{S^c}(\wb^*,\wb;\xb)\big)^2\cdot (\tilde \wb^\top\xb)^2\big]}_{I_9}}\cdot\sqrt{\EE_{\xb}[\ind(S)]}.
\end{align*}
Similarly, let $\bar\xb=(l\cos\phi,l\sin\phi)$ be the projection of $\xb$ onto the 2-dimensional space spanned by $\wb^*$ and $\wb$, we have $\tilde \wb^\top\xb\le l\|\tilde \wb\|_2$. This implies that 
\begin{align*}
I_9&\le \frac{\|\tilde \wb\|_2^2}{2\pi\sigma^2}\int_{0}^\infty p(l)l^3 \dd l\int_{-\pi}^{\pi} \Big(e^{-|l\|\wb\|_2\sin\phi +r|/\sigma}+e^{-|l\|\wb\|_2\sin\phi -r|/\sigma}\Big)^2\dd \phi\notag\\
&\le \frac{\|\tilde \wb\|_2^2}{\pi\sigma^2}\int_{0}^\infty p(l)l^3 \dd l\int_{-\pi}^{\pi} e^{-2|l\|\wb\|_2\sin\phi +r|/\sigma}+e^{-2|l\|\wb\|_2\sin\phi -r|/\sigma}\dd \phi\notag\\
& = \frac{4\|\tilde \wb\|_2^2}{\pi\sigma^2}\int_{0}^\infty p(l)l^3 \dd l\underbrace{\int_{-\pi/2}^{\pi/2} e^{-2|l\|\wb\|_2\sin\phi +r|/\sigma}\dd \phi}_{I_{10}}
\end{align*}
where the second inequality is based on the Young's inequality.
Then we have
\begin{align}\label{eq:upperbound_I10_1}
I_{10} &= \int_{-\pi/2}^{-\pi/3} e^{-2|l\|\wb\|_2\sin\phi +r|/\sigma}\dd \phi+\int_{-\pi/3}^{\pi/3} e^{-2|l\|\wb\|_2\sin\phi +r|/\sigma}\dd \phi+\int_{\pi/3}^{\pi/2} e^{-2|l\|\wb\|_2\sin\phi +r|/\sigma}\dd \phi\notag\\
&\le 2\int_{\pi/3}^{\pi/2} e^{-2|l\|\wb\|_2\sin\phi -r|/\sigma}\dd \phi+\int_{-\pi/3}^{\pi/3} e^{-2|l\|\wb\|_2\sin\phi +r|/\sigma}\dd \phi\notag\\
&\le 2\int_{\pi/3}^{\pi/2} e^{-2|l\|\wb\|_2\sin\phi -r|/\sigma}\dd \phi+2\int_{-\pi/3}^{\pi/3} e^{-2|l\|\wb\|_2\sin\phi +r|/\sigma}\cos\phi\dd \phi,
\end{align}
where the inequality follows from the fact that $|l\|\wb\|_2\sin(-\phi) +r|\le|l\|\wb\|_2\sin\phi +r| $ holds for any $\phi\in[0,\pi/2]$, and the second inequality holds since $2\cos\phi\ge 1$ for any $\phi\in[-\pi/3,\pi/3]$.
Note that the first term on the R.H.S. of \eqref{eq:upperbound_I10_1} is similar to $I_8$, thus we have
\begin{align}\label{eq:bound_firstterm_I10}
\int_{\pi/3}^{\pi/2} e^{-2|l\|\wb\|_2\sin\phi -r|/\sigma}\dd \phi \le \left\{
\begin{array}{ll}
    \pi/6& l\le r/(\sin(\pi/3)\|\wb\|_2) \\
    \frac{2\sigma e^{2r/\sigma}}{l\|\wb\|_2}&l\ge r/(\sin(\pi/3)\|\wb\|_2) 
\end{array}
\right.
\end{align}
Regarding the second term on the R.H.S. of \eqref{eq:upperbound_I10_1}, we have
\begin{align*}
\int_{-\pi/3}^{\pi/3} e^{-2|l\|\wb\|_2\sin\phi +r|/\sigma}\cos\phi\dd \phi&\le \int_{-1}^{1} e^{-2|l\|\wb\|_2z +r|/\sigma}\dd z.
\end{align*}
If $l\|\wb\|_2\le r$, it holds that
\begin{align}\label{eq:bound_I4_l2_case1}
\int_{-1}^1 e^{-2|l\|\wb\|_2z+r|/\sigma}\dd z = \int_{-1}^1 e^{-2(l\|\wb\|_2z+r)/\sigma} \dd z = \frac{\sigma}{2l\|\wb\|_2}\Big(e^{-2(r-l\|\wb\|_2)/\sigma} - e^{-2(r+l\|\wb\|_2)/\sigma}\Big)\le \frac{\sigma}{2l\|\wb\|_2}. 
\end{align}
If $l\|\wb\|_2>r$ we have
\begin{align}\label{eq:bound_I4_l2_case2}
\int_{-1}^1 e^{-2|l\|\wb\|_2z+r|/\sigma}\dd z&=\int_{-1}^{-r/(l\|\wb\|_2)} e^{2(l\|\wb\|_2z+r)/\sigma}\dd z + \int_{-r/(l\|\wb\|_2)}^1 e^{-2(l\|\wb\|_2z+r)/\sigma}\dd z \notag\\
& = \frac{\sigma}{2l\|\wb\|_2}\Big(2-e^{-2(l\|\wb\|-r)/\sigma}-e^{-2(l\|\wb\|+r)/\sigma}\Big)\notag\\
&\le \frac{\sigma}{l\|\wb\|_2}.
\end{align}
Combining the above results we can immediately get $\int_{-\pi/3}^{\pi/3} e^{-2|l\|\wb\|_2\sin\phi +r|/\sigma}\cos\phi\dd \phi\le \sigma/(l\|\wb\|_2)$, which yields the following upper bound on $I_{10}$ by combining with \eqref{eq:bound_firstterm_I10},
\begin{align*}
I_{10}\le \left\{
\begin{array}{ll}
    \pi/3+\frac{2\sigma}{l\|\wb\|_2}& l\le r/(\sin(\pi/3)\|\wb\|_2) \\
    \frac{6\sigma e^{2r/\sigma}}{l\|\wb\|_2}&l\ge r/(\sin(\pi/3)\|\wb\|_2) 
\end{array}
\right.
\end{align*}
Set $\sigma = r$, we have
\begin{align*}
I_{10}\le \frac{\sigma}{l\|\wb\|_2}\cdot\bigg(\frac{\pi}{3\sin(\pi/3)}+2\bigg)\le \frac{6\sigma e^{2r/\sigma}}{l\|\wb\|_2}.
\end{align*}
This further implies the following upper bound on $I_9$,
\begin{align*}
I_9\le \frac{4\|\tilde\wb\|_2^2}{\pi\sigma^2} \int_{0}^\infty \frac{6\sigma e^{2r/\sigma} p(l)l^2}{\|\wb\|_2}\dd l= \frac{48e^2\|\tilde\wb\|_2^2}{\pi\sigma\|\wb\|_2}.
\end{align*}
where the equality is due to the fact that the covariance matrix of $\xb$ is identity.
Note that $\EE_{\xb}[\ind(S)] = \err_{\cD}(\wb^*)$. Then it holds that
\begin{align*}
I_2\le \frac{7e\|\tilde\wb\|_2}{\sqrt{\pi}\|\wb\|_2^{1/2}}\cdot\sqrt{\frac{\err_{\cD}(\wb^*)}{\sigma}}.
\end{align*}

\paragraph{Combining the upper bound of $I_1$ and lower bound of $I_2$.}
Consequently, we have if the angle $\theta$ satisfies
\begin{align*}
\sin\theta\ge \max\bigg\{\frac{4r}{R\|\wb\|_2},\frac{100  U r}{R\|\wb\|_2\sin^{1/2}(\theta') }\bigg\},
\end{align*}
$\wb^{*\top}\nabla L_{\cD}^{p,r}(\wb)$ can be lower bounded by
\begin{align*}
\wb^{*\top}\nabla L_{\cD}^{p,r}(\wb) = I_1 + I_2 \le I_1 + |I_2| = -\frac{3U'R^2\|\tilde\wb\|_2\cdot\sin\theta'\cdot e^{-1}}{5\|\wb\|_2 } + \frac{7 e\|\tilde \wb\|_2}{\sqrt{\pi}\|\wb\|_2^{1/2}}\cdot\sqrt{\frac{\err_\cD(\wb^*)}{\sigma}},
\end{align*}
which further leads  to
\begin{align*}
\wb^{*\top}\nabla L_{\cD}^{p,r}(\wb)\le -\frac{U'R^2\|\tilde\wb\|_2\cdot\sin\theta'\cdot e^{-1}}{2\|\wb\|_2 }
\end{align*}
if we have
\begin{align*}
\err_{\cD}(\wb^*)\le \frac{U'^2\sigma\sin^2\theta'}{2^{14}\|\wb\|_2R^4} = \frac{U'^2r\sin^2\theta'}{2^{14}\|\wb\|_2R^4}.
\end{align*}
This completes the proof
\end{proof}

\subsection{Proof for Lemma \ref{lemma:convergence_guarantee_psat}}
\label{sec:convergence_guarantee_psat}

We will decompose the entire into two parts: (1) proving a lower bound of the angle $\theta'(\wb)$; and (2) establishing a general convergence guarantee for Algorithm \ref{alg:projected_advtraining}.

In terms of the first part, we summarize the lower bound of $\theta'(\wb)$ in the following lemma.
\begin{figure}
\centering
\includegraphics[width=0.3\textwidth]{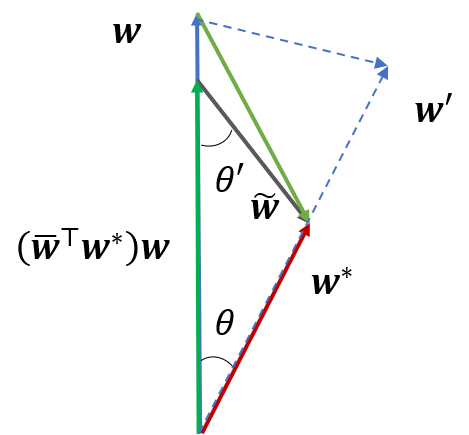}
\caption{Illustration of $\theta$ and $\theta'$. \label{fig:diagram_angle}}
\end{figure}
\begin{lemma}\label{lemma:lowerbound_theta'}
Let $\theta = \angle(\wb,\wb^*)$ and $\theta' = \angle(-\wb, \tilde\wb)$, we have $\sin\theta'\ge \sqrt{2}/2$ if $\|\wb\|_2< \|\wb^*\|_2$ and $\sin\theta'\ge\frac{1}{10\|\wb\|_2 d^{1/2-1/p}}$ if $\|\wb\|_2\ge \|\wb^*\|_2$.
\end{lemma}
\begin{proof}[Proof of Lemma \ref{lemma:lowerbound_theta'}]

Without loss of generality we consider the case that $\theta\in(0,\pi/2)$, if $\theta\in(\pi/2,\pi)$  we can simply apply $\wb\rightarrow -\wb$. In particular, let $\vb = -\wb$ and $\tilde\vb=\wb^*-(\bar\vb\wb^*)\vb$, it is easy to see that
\begin{align*}
-\la\wb, \wb^*-(\bar\wb^\top\wb^*)\wb\ra = \la\vb, \wb^*-(\bar\vb^\top\wb^*)\vb\ra,
\end{align*}
implying that $\sin(\angle(-\vb,\tilde\vb))=\sin(\angle(-\wb,\wb^*))$.
Recall that $\tilde \wb = \wb^* - (\bar\wb^\top\wb^*)\wb$. First note that if we have $\bar\wb^\top\wb^*\le 0$, it is easy to see that $-\wb^\top\tilde\wb = -\wb^{*\top}\wb +  (\bar\wb^\top\wb^*)\|\wb\|_2^2\le 0$ since we have $\theta<\pi/2$, which implies that $\theta'\ge \pi/2$. In the sequel we will focus on the case of $\bar\wb^\top\wb^*\ge 0$. Note that we have $\|\wb\|_q = \|\wb^*\|_q = 1$, which yields $\bar\wb^\top\wb^*\le 1$. 
Therefore, as shown in Figure \ref{fig:diagram_angle}, it is clear that
\begin{align*}
\theta' = \angle(-\wb, \tilde\wb) = \angle (-\wb,\wb^*-\wb) + \angle(\wb^*-\wb, \tilde\wb) \ge \angle (-\wb,\wb^*-\wb):=\tilde\theta'.
\end{align*}
By Sine formula, it is easy to see if $\|\wb\|_2\le \|\wb^*\|_2$, we have $\angle (-\wb,\wb^*-\wb)\ge \angle (-\wb^*,\wb^*-\wb)$ and thus $\angle (-\wb,\wb^*-\wb)\ge (\pi-\theta)/2\ge \pi/4$, which completes the proof of the first case.

If $\|\wb\|_2\le \|\wb^*\|_2$, we define $\wb' = \|\wb\|_2\wb^*/\|\wb^*\|_2$, as displayed in Figure \ref{fig:diagram_angle}, which clearly satisfies $\|\wb'\|_2=\|\wb\|_2\ge \|\wb^*\|_2$. Then we will upper bound the length of $\tilde\wb' = \wb^*-\wb$. 
By triangle inequality, it holds that \begin{align}\label{eq:bound_difference}
\|\tilde\wb'\|_2  = \|\wb - \wb^*\|_2 = \|\wb-\wb'+\wb'-\wb^*\|_2\le \|\wb-\wb'\|_2 + \|\wb'-\wb^*\|_2.
\end{align}
Note that $\wb$ and $\wb'$ have the same length and $\angle (\wb, \wb')=\theta$, we have
\begin{align}\label{eq:bound_difference1}
\|\wb-\wb'\|_2\le \|\wb\|_2\theta.
\end{align}
Additionally, note that $\wb'$ and $\wb^*$ are parallel and $\|\wb^*\|_q=1$, we have
\begin{align*}
\wb^* = \wb'/\|\wb'\|_q,
\end{align*}
and
\begin{align}\label{eq:bound_difference2}
\frac{\|\wb'-\wb^*\|_2}{\|\wb'\|_2} = \frac{\|\wb'-\wb^*\|_q}{\|\wb'\|_q} =(1 - 1/\|\wb'\|_q).
\end{align}
Plugging \eqref{eq:bound_difference1} and \eqref{eq:bound_difference2} into \eqref{eq:bound_difference} yields
\begin{align*}
\|\tilde\wb'\|_2&\le \|\wb-\wb'\|_2+\|\wb'-\wb^*\|_2\notag\\
&\|\wb\|_2 \theta+ \|\wb'\|_2\cdot\bigg(1-\frac{1}{\|\wb'\|_q}\bigg)\notag\\
&\le \|\wb\|_2\cdot\bigg(\theta + 1 - \frac{1}{\|\wb'\|_q}\bigg),
\end{align*}
where in the last inequality we use the fact that $\|\wb'\|_2 = \|\wb\|_2$. 
Then note that $\|\wb^*\|_2 = \|\wb\|_2/\|\wb'\|_q$, by Sine formula, we have
\begin{align*}
\sin\theta' \ge \tilde\theta' = \frac{\sin\theta\|\wb^*\|_2}{\|\tilde\wb'\|_2} \ge \frac{\sin\theta /\|\wb\|_q}{\theta+1-1/\|\wb'\|_q} = \frac{\sin\theta}{(\theta+1)\|\wb'\|_q-1}.
\end{align*}
Note that $\|\wb\|_q = 1$, by triangle inequality we have
\begin{align*}
\|\wb'\|_q \le \|\wb\|_q +\|\wb - \wb'\|_q \le 1 + \|\wb-\wb'\|_2d^{1/q-1/2} \le 1 + \|\wb\|_2\theta d^{1/q-1/2}.
\end{align*}
This further implies that 
\begin{align*}
(\theta+1)\|\wb'\|_q-1\le \theta (1 + \|\wb\|_2 d^{1/q-1/2}) + \|\wb\|_2 d^{1/q-1/2}\theta^2\le (1+\pi/2)\theta (1 + \|\wb\|_2 d^{1/q-1/2}),
\end{align*}
where the last inequality holds since $\theta\le \pi/2\le 2$. Then note that $\sin\theta/\theta\ge 2/\pi$ for any $\theta\in(0,\pi/2)$, we have
\begin{align*}
\sin\theta'\ge \frac{\sin\theta}{(\theta+1)\|\wb'\|_q-1}\ge \frac{1}{5(1+\|\wb\|_2 d^{1/q-1/2})}\ge \frac{1}{10\|\wb\|_2 d^{1/q-1/2}},
\end{align*}
where the last inequality holds since $\|\wb\|_2\ge d^{1/2-1/q}$. Note that this bound also holds for other cases, we are able to complete the proof.
\end{proof}

Then we provide the following lemma that gives the convergence guarantee of Algorithm \ref{alg:projected_advtraining} if for an arbitrary set $\cG\in\cS_q^{d-1}$ we have sufficiently negative $\wb^{*\top}\nabla L_\cD^{p,r}(\wb)$ for any $\wb\in\cG$.
\begin{lemma}\label{lemma:convergence_pssat}
Let $\cG$ be a non-empty subset of $\cS_q^{d-1}$. Assume $r\le 1$ and $\wb^{*\top} \nabla L_\cD^{p,r}(\wb)\le -\epsilon$ for any $\wb\in\cG$, then set $\eta = \epsilon \delta \sigma^2d^{-1}/32$ and $K = 64d\|\wb_1 - \wb^*\|_2^2\delta^{-2}\sigma^{-2}\epsilon^{-2}$, with probability at least $1-\delta$, running Algorithm $\mathsf{PSAT}(p, r)$ for $K$ iterations can find a model $\wb_{k^*}$ with $k^*\le K$ such that $\wb_{k^*}\in\cG^c$.
\end{lemma}
\begin{proof}
[Proof of Lemma \ref{lemma:convergence_pssat}]
We focus on the quantity $\|\wb_k-\wb^*\|_2^2$. In particular, note that the $\ell_q$ ball is a convex set and the gradient $\nabla \ell\big(y_k\wb_k^\top (\xb_i+\bdelta_i^{(k)})/\|\wb_k\|_q)$ is orthogonal to $\wb_k$, we must have $\hat\wb_{k+1}$ stays outside the unit $\ell_q$ ball since  $\|\wb_k\|_q=1$. This further implies that $\wb_{k+1}$ is also the projection of $\hat\wb_{k+1}$ onto the unit $\ell_q$ ball. Then we have
\begin{align*}
\|\wb_{k+1}-\wb^*\|_2^2&\le \|\hat\wb_{k+1}-\wb^*\|_2^2\notag\\
& = \|\wb_k-\wb^*\|_2^2 - 2\eta\bigg\la\wb_k - \wb^*, \nabla \ell\bigg(\frac{y_k\wb_k^\top (\xb_k+\deb_k)}{\|\wb_k\|_q}\bigg)\bigg\ra + \eta^2\bigg\|\nabla \ell\bigg(\frac{y_k\wb_k^\top (\xb_k+\deb_k)}{\|\wb_k\|_q}\bigg)\bigg\|_2^2\notag\\
& = \|\wb_{t}-\wb^*\|_2^2 + 2\eta \wb^{*\top} \nabla \ell\bigg(\frac{y_k\wb_k^\top (\xb_k+\deb_k)}{\|\wb_k\|_q}\bigg) + \eta^2\bigg\|\nabla \ell\bigg(\frac{y_k\wb_k^\top (\xb_k+\deb_k)}{\|\wb_k\|_q}\bigg)\bigg\|_2^2\notag,
\end{align*}
where in the second equality we use the fact that the gradient $\nabla \ell\big(y_k\wb_k^\top (\xb_k+\deb_k)/\|\wb_k\|_q\bigg)$ is orthogonal to $\wb_k$. Then taking expectation over $(\xb_k,y_k)$ conditioned on $\wb_k$, we have \begin{align}\label{eq:recursion_dist}
\EE\big[\|\wb_{k+1}-\wb^*\|_2^2|\wb_k] =  \|\wb_k-\wb^*\|_2^2 + 2\eta \wb^{*\top} \nabla L_{\cD}^{p,r}(\wb_k) + \eta^2\EE\bigg[\bigg\|\nabla \ell\bigg(\frac{y_k\wb_k^\top (\xb_k+\deb_k)}{\|\wb_k\|_q}\bigg)\bigg\|_2^2\bigg|\wb_k\bigg].
\end{align}
Recall that 
\begin{align*}
\nabla \ell\bigg(\frac{y\wb^\top (\xb+\deb)}{\|\wb\|_q}\bigg) &= \ell'(yh(\wb,\xb+\deb))\cdot y\cdot \nabla_{\wb}h(\wb,\xb+\deb)\notag\\
& = \frac{e^{-|y\wb^\top\xb/\|\wb\|_q-r|/\sigma}}{\sigma}\cdot y\cdot \bigg(\Ib-\frac{\bar\wb\wb^\top}{\|\wb\|_q^q}\bigg)\frac{\xb+\deb}{\|\wb\|_q}.
\end{align*}
If $\|\wb\|_q = 1$, we have
\begin{align*}
\bigg\|\nabla \ell\bigg(\frac{y\wb^\top (\xb+\deb)}{\|\wb\|_q}\bigg)\bigg\|_2\le \frac{\big\|(\Ib-\bar\wb\wb^\top)(\xb+\deb)\big\|_2}{\sigma}\le \frac{2\|\xb+\deb\|_2}{\sigma}\le \frac{2(\|\xb\|_2 + \|\bdelta\|_2)}{\sigma},
\end{align*}
where the first inequality is due to $e^{-|y\wb^\top\xb/\|\wb\|_q-r|}/\le 1$ and the second inequality is due to $\|\bar\wb\wb^\top\|_2=|\bar\wb^\top\wb|=1$. Since $p\ge 2$, we have $\|\bdelta\|_2\le d^{1/2-1/p} r$. Thus it holds that
\begin{align*}
\bigg\|\nabla \ell\bigg(\frac{y\wb^\top (\xb+\deb)}{\|\wb\|_q}\bigg)\bigg\|_2\le \frac{2(\|\xb\|_2 + d^{1/2-1/p} r)}{\sigma}.
\end{align*}
Further note that the covariance matrix of $\xb$ is the identity matrix, based on Young's inequality and the assumption that $r\le 1$, we have
\begin{align*}
\EE\bigg[\bigg\|\nabla \ell\bigg(\frac{y_k\wb_k^\top (\xb_k+\deb_k)}{\|\wb_k\|_q}\bigg)\bigg\|_2^2\bigg|\wb_k\bigg]\le \frac{8}{\sigma^2}\EE_{\xb\in\cD}[\|\xb\|_2^2+d^{1-2/p} r^2] = \frac{8}{\sigma^2}(d+d^{1-2/p}r^2)\le \frac{16d}{\sigma^2}.
\end{align*}
Plugging the above inequality into \eqref{eq:recursion_dist} gives
\begin{align}\label{eq:recursion_dist_use_gradbound}
\EE\big[\|\wb_{k+1}-\wb^*\|_2^2|\wb_k] \le  \|\wb_k-\wb^*\|_2^2 + 2\eta \wb^{*\top} \nabla L_{\cD}^{p,r}(\wb_k) + \frac{16\eta^2 d}{\sigma^2}. 
\end{align}
Note that if $\wb_k\in\cG$ we have $\wb^{*\top} \nabla L_{\cD}^{p,r}(\wb_k)\le -\epsilon$. 
Then we can denote $\mathfrak{E}_k$ as the event that  $\wb_s\in\cG$ for all $s\le k$, which further leads to $L_{\cD}^{p,r}(\wb_k)\cdot\ind(\mathfrak{E}_k)\le -\epsilon \ind(\mathfrak{E}_k)$. Thus multiply by $\ind(\mathfrak{E}_k)$ on both sides of \eqref{eq:recursion_dist_use_gradbound} gives us
\begin{align*}
\EE\big[\|\wb_{k+1}-\wb^*\|_2^2\cdot\ind(\mathfrak{E}_k)|\wb_k] \le  \|\wb_k-\wb^*\|_2^2\cdot\ind(\mathfrak{E}_k) - \eta\epsilon\cdot\ind(\mathfrak{E}_k) + \frac{16\eta^2 d}{\sigma^2}.
\end{align*}
Note that $\mathfrak{E}_{k+1}\subset\mathfrak{E}_k$, we have $\ind(\mathfrak{E}_{k+1})\le \ind(\mathfrak{E}_k)$, which implies that
\begin{align*}
\EE\big[\|\wb_{k+1}-\wb^*\|_2^2\cdot\ind(\mathfrak{E}_{k+1})|\wb_k] \le  \|\wb_k-\wb^*\|_2^2\cdot\ind(\mathfrak{E}_k) - \eta\epsilon\cdot\ind(\mathfrak{E}_k) + \frac{16\eta^2 d}{\sigma^2}.
\end{align*}
Therefore, taking a total expectation and applying summation from $k=0$ to $k=K-1$, we can get
\begin{align*}
\EE\big[\|\wb_{K}-\wb^*\|_2^2\cdot\ind(\mathfrak{E}_{K})] \le \|\wb_1-\wb^*\|_2^2\cdot\ind(\mathfrak{E}_{K}) - \eta\epsilon\cdot \sum_{s=1}^{K}\EE[\ind(\mathfrak{E}_K)] + \frac{16K\eta^2 d}{\sigma^2}.
\end{align*}
Dividing by $K$ on both sides and rearranging terms, we obtain
\begin{align*}
\frac{1}{K}\sum_{s=1}^{K}\EE[\ind(\mathfrak{E}_k)]\le \frac{1}{\epsilon}\cdot\bigg(\frac{\|\wb_1-\wb^*\|_2^2}{K\eta} + \frac{16\eta d}{\sigma^2}\bigg).
\end{align*}
Then we can set 
\begin{align*}
\eta = \frac{\epsilon\delta \sigma^2}{32d},\quad \mbox{and}\quad K = \frac{64 d\|\wb_1-\wb^*\|_2^2}{\sigma^2\delta^2\epsilon^2 }
\end{align*}
such that
\begin{align*}
\frac{1}{K}\sum_{s=1}^K\EE[\ind(\mathfrak{E}_K)]\le \delta.
\end{align*}
Then by Markov inequality we have with probability at least $1-\delta$,
\begin{align*}
\frac{1}{K}\sum_{s=1}^K\ind(\mathfrak{E}_K)< 1.
\end{align*}
Therefore we immediately have $\ind(\mathfrak{E}_K)< 1$ since $\{\ind(\mathfrak{E}_k)\}_{k=1,\dots,K}$ is non-increasing, which implying that there exists a $k^*\le K$ such that  $\wb_{k^*}\in\cG^c$.

\end{proof}

\begin{proof}[Proof of Lemma \ref{lemma:convergence_guarantee_psat}]
By Lemma \ref{lemma:lowerbound_theta'} we can get
\begin{align*}
\sin(\theta'(\wb))\ge\left\{
\begin{array}{ll}
    \frac{1}{10\|\wb\|_2d^{1/p-1/2}} & \|\wb\|_2\ge \|\wb^*\|_2 \\
    \frac{\sqrt{2}}{2} & \|\wb\|_2<\|\wb^*\|_2 
\end{array}
\right.
\end{align*}
Note that we have
\begin{align*}
r= O(d^{\frac{3}{2p}-\frac{3}{4}})\le O\big(\|\wb\|_2\sin^{3/2}(\theta'(\wb))\big)
\end{align*}
and
\begin{align*}
\err_{\cD}(\wb^*) = O(r d^{2/p-1}) \le O\big(r\|\wb\|_2^{-1}\sin^{2}(\theta'(\wb))\big)
\end{align*}
since $\|\wb\|_2\le 1$. Therefore all conditions in Lemma \ref{lemma:upperbound_opt_grad_product_main} can be satisfied and its argument can be applied. In particular, we have if
\begin{align*}
\sin(\theta(\wb)) \ge \left\{
\begin{array}{ll}
   \max\Big\{\frac{4r}{R\|\wb\|_2}, \frac{400Urd^{1/(2p)-1/4}}{R\|\wb\|_2^{1/2} }\Big\},  & \|\wb\|_2\ge \|\wb^*\|_2  \\
    \max\Big\{\frac{4r}{R\|\wb\|_2}, \frac{200\sqrt{2}Ur}{R\|\wb\|_2 }\Big\} & \|\wb\|_2<\|\wb^*\|_2.
\end{array}
\right.
\end{align*}
we have the following upper bound on the inner product $\wb^{*\top} \nabla L_{\cD}^{p,r}(\wb)$
\begin{align*}
\wb^{*\top} \nabla L_{\cD}^{p,r}(\wb)&\le -\frac{U'R^2\|\tilde\wb\|_2\sin(\theta(\wb)') e^{-1}}{2\|\wb\|_2} \notag\\
&= -\frac{U'R^2 e^{-1}\|\wb^*\|_2\sin(\theta(\wb))}{2\|\wb\|_2},  
\end{align*}
where the equality is by Sine rule. Then by Lemma \ref{lemma:convergence_pssat} we can set the step size as $\eta = O\big(\delta r^3 d^{\frac{1}{2p}-\frac{1}{4}}\big)$, then with probability at least $1-\delta$, the algorithm $\mathsf{PSAT}(p, r)$ can find a model $\wb_{k^*}$ such that 
\begin{align}\label{eq:condition_wk*}
\sin(\theta(\wb_{k^*})) \le \left\{
\begin{array}{ll}
   \max\Big\{\frac{4r}{R\|\wb\|_2}, \frac{400Urd^{1/(2p)-1/4}}{R\|\wb\|_2^{1/2} }\Big\},  & \|\wb\|_2\ge \|\wb^*\|_2  \\
    \max\Big\{\frac{4r}{R\|\wb\|_2}, \frac{200\sqrt{2}Ur}{R\|\wb\|_2 }\Big\} & \|\wb\|_2<\|\wb^*\|_2.
\end{array}
\right.
\end{align}
within $K=O\big(d\|\wb_1-\wb^*\|_2^2\delta^{-2}r^{-4}d^{\frac{1}{2}-\frac{1}{p}}\big)$ iterations.

\end{proof}

\begin{figure}
\centering
\includegraphics[width=0.3\textwidth]{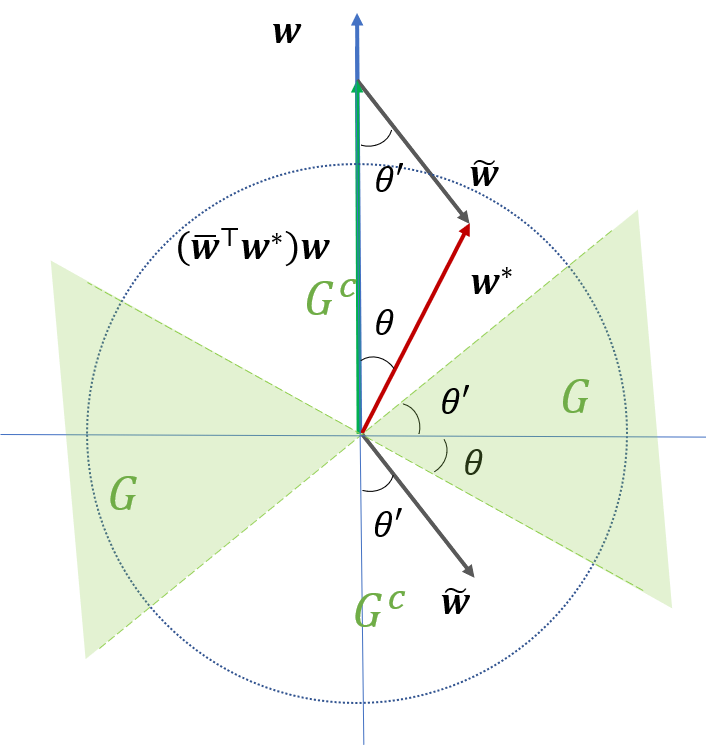}
\caption{Illustration of the set $G$ and $G^c$. \label{fig:diagram}}
\end{figure}
\subsection{Proof of Theorem \ref{thm:guarantee_normalized_classifier}}\label{sec:proof_theorem_normalized}
\begin{proof}[Proof of Theorem \ref{thm:guarantee_normalized_classifier}]

We are going to show that the condition \eqref{eq:condition_wk*} can imply the desired bound on $\err_{\cD}^{p,r}(\wb_{k^*})$.
Note that the optimal robust error can be written as
\begin{align*}
\OPT_{p,r} = \EE_{(\xb,y)\sim\cD}\big[\ind(y\wb^{*\top}\xb\le r)\big]. 
\end{align*}
Note that the robust error consists of two classes of data: (1) the data satisfies  $|\wb^{*\top}\xb|\le r$; and (2) the data satisfies  $|\wb^{*\top}\xb|> r$ and $y\wb^{*\top}\xb<0$. Therefore, we can get lower and upper bounds on the optimal robust error $\OPT_{p,r}$ as follows,
\begin{align}\label{eq:bounds_err_w*}
\opt_{p,r}&\ge\EE_{\xb\sim\cD_x}\big[\ind(|\wb^{*\top}\xb|\le r)\big]\notag\\
\opt_{p,r}&\le \EE_{\xb\sim\cD_x}\big[\ind(|\wb^{*\top}\xb|\le r)\big] + \err_{\cD}(\wb^*).
\end{align}
By Assumption \ref{assump:isotropic}, we have the data distribution $\cD_x$ satisfies $U$-anti-concentration and $(U',R)$ anti-anti-concentration with $U,R$ being constants. Therefore, it follows that $\EE_{\xb\sim\cD_x}\big[\ind(|\wb^{*\top}\xb|\le r)\big]=\Theta(r\|\wb^*\|_2^{-1})$ since we have $r=O(d^{\frac{3}{2p}-\frac{3}{4}})\le Rd^{1/p-1/2}$.
Besides, note that we also have $\err_\cD(\wb^*)=O(rd^{2/p-1})\le O(r\|\wb^*\|_2^{-1})$ due to our assumption. Therefore, it is clear that $\OPT_{p,r} = \Theta(r\|\wb^*\|_2^{-1})$. Moreover, regarding $\wb_{k^*}$ we can similarly get that
\begin{align}\label{eq:bounds_err_wk}
\err_\cD^{p,r}(\wb_{k^*})\le \EE_{\xb\sim\cD_x}\big[\ind(|\wb_{k^*}^{\top}\xb|\le r)\big] + \err_{\cD}(\wb_{k^*}).
\end{align}
Clearly due to Assumption \ref{assump:isotropic} we have $\EE_{\xb\sim\cD_x}\big[\ind(|\wb_{k^*}^{\top}\xb|\le r)\big] = \Theta(r\|\wb_{k^*}\|_2^{-1})$. Additionally, we have the following regarding $|\err_\cD(\wb_{k^*}) - \err_\cD(\wb^*)|$,
\begin{align}\label{eq:upperbound_error_diff_w_w*}
|\err_\cD(\wb_{k^*}) - \err_\cD(\wb^*)| &= \EE[|\ind(y\wb_{k^*}^\top\xb\le 0) - \ind(y\wb^{*\top}\xb\le0)|]\notag\\
& = \EE[|[\ind(y\wb_{k^*}^\top\xb\le 0) - \ind(y\wb^{*\top}\xb\le0)]\cdot[\ind(y\wb^{*\top}\xb\le0) + \ind(y\wb^{*\top}\xb\ge0)]|]\notag\\
&\le \EE[|[\ind(y\wb_{k^*}^\top\xb\le 0) - \ind(y\wb^{*\top}\xb\le0)]\cdot\ind(y\wb^{*\top}\xb\le0)|]\notag\\
&\qquad + \EE[|[\ind(y\wb_{k^*}^\top\xb\le 0) - \ind(y\wb^{*\top}\xb\le0)]\cdot \ind(y\wb^{*\top}\xb\ge0)|].
\end{align}
where the we use the triangle inequality in the last line. Moreover, note that
\begin{align}\label{eq:upperbound_error_diff_w_w*_term1}
\EE[|[\ind(y\wb_{k^*}^\top\xb\le 0) - \ind(y\wb^{*\top}\xb\le0)]\cdot\ind(y\wb^{*\top}\xb\le0)|]\le \EE[\ind(y\wb^{*\top}\xb\le0)] = \err_{\cD}(\wb^*),
\end{align}
and
\begin{align}\label{eq:upperbound_error_diff_w_w*_term2}
\EE[|[\ind(y\wb_{k^*}^\top\xb\le 0) - \ind(y\wb^{*\top}\xb\le0)]\cdot \ind(y\wb^{*\top}\xb\ge0)|] &= \EE[\ind(y\wb_{k^*}^\top\xb\le 0) \cdot \ind(y\wb^{*\top}\xb\ge0)]\notag\\
& = \EE[\ind(\text{sgn}(\wb_{k^*}^\top\xb)\neq\text{sgn}(\wb_{k^*}^\top\xb))].
\end{align}
By Claim 3.4 in \citet{diakonikolas2020nonconvex} we have
\begin{align}\label{eq:upperbound_difference_w_w*}
\EE[\ind(\text{sgn}(\wb_{k^*}^\top\xb)\neq\text{sgn}(\wb^{*\top}\xb))]=\Theta(\angle(\wb_{k^*}, \wb^*)) = O(\theta^*) =  O(rd^{1/(2p)-1/4}\|\wb_{k^*}\|_2^{-1/2}).
\end{align}
Using \eqref{eq:upperbound_difference_w_w*} and combining the bounds in \eqref{eq:upperbound_error_diff_w_w*_term1} and \eqref{eq:upperbound_error_diff_w_w*_term2}, the following holds by \eqref{eq:upperbound_error_diff_w_w*}
\begin{align}\label{eq:bound_clean_err_w}
\err_\cD(\wb_{k^*})\le \err_\cD(\wb^*)+|\err_\cD(\wb_{k^*}) - \err_\cD(\wb^*)|= O(rd^{1/(2p)-1/4}\|\wb_{k^*}\|_2^{-1/2}),
\end{align}
where the first inequality is due to triangle inequality and in the equality we use the fact that $\err_\cD(\wb^*) = O(rd^{2/p-1})$ and $\|\wb_{k^*}\|_2\ge d^{1/p-1/2}$. 

Therefore, it remains to lower bound the norm of $\|\wb_{k^*}\|_2$. Note that we only need to consider the case that $\|\wb_{k^*}\|_2< \|\wb^*\|_2$ since otherwise we can directly use $\|\wb^*\|_2$ as a lower bound of $\|\wb_{k^*}\|_2$.

When $\|\wb_{k^*}\|_2< \|\wb^*\|_2$, by \eqref{eq:condition_wk*} we can get that $\theta=\Theta(r\|\wb_{k^*}\|_2^{-1})$. Then we define $\wb'=\wb_{k^*} \|\wb^*\|_2/\|\wb_{k^*}\|_2$. 
Then we can see that $\wb'$ and $\wb_{k^*}$ are parallel, which implies that
\begin{align*}
\frac{\|\wb'\|_2}{\|\wb_{k^*}\|_2} = \frac{\|\wb'\|_q}{\|\wb_{k^*}\|_q} = \|\wb'\|_q,
\end{align*}
where the second equality is due to $\|\wb_{k^*}\|_q=1$. Note that $\|\wb^*\|_q=1$, by triangle inequality we have
\begin{align*}
\|\wb'\|_q\le \|\wb^*\|_q + \|\wb^*-\wb'\|_q \le 1 + \|\wb^*-\wb'\|_2 d^{1/p-1/2}\le 1 + \|\wb^*\|_2\theta d^{1/p-1/2},
\end{align*}
where the second inequality holds since $\|\zb\|_q\le \|\zb\|_2 d^{1/2-1/p}$ (where we use the fact that $1/p+1/q=1$) and the last inequality is due to $\|\wb^*\|_2 = \|\wb'\|_2$. Consequently, we can get
\begin{align}\label{eq:lowerbound_norm_wk1}
\|\wb_{k^*}\|_2 = \frac{\|\wb'\|_2}{\|\wb'\|_q}\ge \frac{\|\wb^*\|_2}{1 + \|\wb^*\|_2\theta d^{1/2-1/p}}.
\end{align}
Note that we have $\theta^* = \Theta(r\|\wb_{k^*}\|_2^{-1})$, we immediately have $\theta = O(r\|\wb_{k^*}\|_2^{-1})$. Note that the perturbation level satisfies $r=O(d^{\frac{3}{2p}-\frac{3}{4}}) \le c\cdot d^{1/2-1/p}$ for some sufficiently small constant $c$, we can get $\theta d^{1/p-1/2}\le 0.5 \|\wb_{k^*}\|_2^{-1}$.  Plugging this into \eqref{eq:lowerbound_norm_wk1} and use the fact that gives
\begin{align*}
\|\wb_{k^*}\|_2\ge \frac{\|\wb^*\|_2}{1+0.5\|\wb^*\|_2/\|\wb_{k^*}\|_2 },
\end{align*}
which implies that $\|\wb_{k^*}\|_2\ge 0.5\|\wb^*\|_2$.
Plugging this into \eqref{eq:bound_clean_err_w} further gives
\begin{align*}
\err_\cD(\wb_{k^*}) = O(rd^{1/(2p)-1/4}\|\wb^*\|_2^{-1/2}).
\end{align*}
Note that we also have $\EE_{\xb\sim\cD_x}\big[\ind(|\wb_{k^*}^{\top}\xb|\le r)\big] =\Theta(r\|\wb_{k^*}\|_2^{-1}) = \Theta(r\|\wb_*\|_2^{-1})$. Combining these two bounds into \eqref{eq:bounds_err_wk} we can get the robust error for $\wb_{k^*}$ as follows,
\begin{align*}
\err_\cD^{p,r}(\wb_{k^*})\le \EE_{\xb\sim\cD_x}\big[\ind(|\wb_{k^*}^{\top}\xb|\le r)\big] + \err_{\cD}(\wb_{k^*}) = O(rd^{1/(2p)-1/4}\|\wb^*\|_2^{-1/2})
\end{align*}
where we use the fact that $\|\wb^*\|_2\ge d^{1/p-1/2}$.
Applying the fact that $\OPT_{p,r} = \Theta(r\|\wb^*\|_2^{-1})$ further gives
\begin{align*}
\err_\cD^{p,r}(\wb_{k^*}) = O\big(d^{1/(2p)-1/4}\|\wb_{k^*}\|_2^{1/2}\OPT_{p,r}\big)
\end{align*}
which completes the proof.

\end{proof}

\bibliography{deeplearningreference}

\end{document}